\documentclass[letterpaper]{article} 
\usepackage{aaai24forArXiv}  
\usepackage{times}  
\usepackage{helvet}  
\usepackage{courier}  
\usepackage[hyphens]{url}  
\usepackage{graphicx} 
\usepackage{natbib}  
\usepackage{caption} 
\usepackage{algorithm}
\usepackage{algorithmic}

\usepackage{newfloat}
\usepackage{listings}
\DeclareCaptionStyle{ruled}{labelfont=normalfont,labelsep=colon,strut=off} 
\floatstyle{ruled}
\newfloat{listing}{tb}{lst}{}
\floatname{listing}{Listing}
\title{PAC-Bayes Generalisation Bounds\\ for Dynamical Systems Including Stable RNNs}
\author {
    Deividas Eringis\textsuperscript{\rm 1},
    John Leth\textsuperscript{\rm 1},
    Zheng-Hua Tan\textsuperscript{\rm 1},
    Rafal Wisniewski\textsuperscript{\rm 1},
    Mih\'aly Petreczky\textsuperscript{\rm 2}
}
\affiliations {
    \textsuperscript{\rm 1}Department of Electronic Systems, Aalborg University, Denmark\\
    \textsuperscript{\rm 2}Laboratoire Signal et Automatique de Lille (CRIStAL), France\\
    \{der,jjl,zt,raf\}@es.aau.dk, mihaly.petreczky@centralelille.fr
}

\usepackage{amsmath,amsfonts,bm}

\def\eqref#1{equation~\ref{#1}}

\def\1{\bm{1}}
\newcommand{\train}{\mathcal{D}}

\def\rvepsilon{{\mathbf{\epsilon}}}

\def\rve{{\mathbf{e}}}

\def\rvo{{\mathbf{o}}}

\def\rvq{{\mathbf{q}}}

\def\rvs{{\mathbf{s}}}

\def\rvv{{\mathbf{v}}}

\def\rvx{{\mathbf{x}}}
\def\rvy{{\mathbf{y}}}
\def\rvz{{\mathbf{z}}}

\def\vs{{\bm{s}}}

\def\vv{{\bm{v}}}
\def\vw{{\bm{w}}}
\def\vx{{\bm{x}}}
\def\vy{{\bm{y}}}
\def\vz{{\bm{z}}}

\DeclareMathAlphabet{\mathsfit}{\encodingdefault}{\sfdefault}{m}{sl}
\SetMathAlphabet{\mathsfit}{bold}{\encodingdefault}{\sfdefault}{bx}{n}

\def\sQ{{\mathbb{Q}}}
\def\sR{{\mathbb{R}}}
\def\sS{{\mathbb{S}}}

\def\sU{{\mathbb{U}}}
\def\sV{{\mathbb{V}}}

\def\sX{{\mathbb{X}}}
\def\sY{{\mathbb{Y}}}
\def\sZ{{\mathbb{Z}}}

\newcommand{\E}{\mathbb{E}}

\newcommand{\KL}{D_{\mathrm{KL}}}

\newcommand{\normlone}{L^1}

\usepackage{url}

\usepackage{amssymb}
\usepackage{graphicx}
\usepackage{xcolor}         
\usepackage{booktabs,multirow}

\newcommand{\reals}{\mathbb{R}}

\newcommand{\bP}{\mathbf{P}}
\newcommand{\F}{\mathbf{F}}

\usepackage{amsmath}
\usepackage{amssymb}
\usepackage{amsthm}

\newtheorem{Assumption}{Assumption}[section]
\newtheorem{lemma}{Lemma}[section]
\newtheorem{corollary}{Corollary}[section]

\newtheorem{Remark}{Remark}[section]
\newtheorem{Definition}{Definition}[section]
\newtheorem{Theorem}{Theorem}[section]
\newtheorem{proposition}{Proposition}[section]

\usepackage{empheq}

\newcommand{\class}{\(\mathcal{S}\)}
\newcommand{\aseq}{\overset{\text{a.s.}}{=}}

\newcommand{\deividas}[1]{#1} 
\newcommand{\mihaly}[1]{ #1}

\begin{document}

\maketitle

\begin{abstract}
In this paper, we derive a PAC-Bayes bound on the generalisation gap, in a supervised time-series setting for a special class of discrete-time non-linear dynamical systems. This class includes stable recurrent neural networks (RNN), and the motivation for this work was its application to RNNs. In order to achieve the results, we impose some stability constraints, on the allowed models. 
Here, stability is understood in the sense of dynamical systems. For RNNs, these stability conditions can be expressed in terms of conditions on the weights. 
We assume the processes involved are essentially bounded and the loss functions are Lipschitz. The proposed bound on the generalisation gap depends on the mixing coefficient of the data distribution, and the essential supremum of the data. Furthermore, the bound converges to zero as the dataset size increases.
In this paper, we 1) formalize the learning problem, 2) derive a PAC-Bayesian error bound for such systems, 3) discuss various consequences of this error bound, and 4) show an illustrative example, with discussions on computing the proposed bound. Unlike other available bounds  the derived bound holds for non i.i.d. data (time-series) and it does not grow with the number of steps of the RNN.

\end{abstract}
\section{Introduction}
The Probably Approximately Correct (PAC)-Bayesian framework 
has been a popular tool for obtaining generalisation bounds and to derive efficient learning algorithms, see \cite{alquier2021userfriendly,Dziugaite2017}.
\par
\textbf{Contribution.}
     In this paper we develop PAC-Bayesian inequalities  for a class of
     discrete-time dynamical systems with hidden (unobserved) states. 
    This class includes a wide variety of dynamical systems, ranging from linear time-invariant state-space representations (LTIs) to recurrent neural networks (RNNs).
     We view dynamical systems as hypotheses (predictors) which 
     transform
     past inputs and outputs (labels) to estimates of the current output (label).
     That is, our framework captures both time-series forecasting and learning
     models
     which causally transform sequences of inputs to sequences of outputs.
     In the latter case, the dynamical system at hand uses only past inputs. 
     Furthermore, training data represents a single time-series sampled from 
     the input and output processes. That is, the training data is not i.i.d.
\par
     The PAC-Bayesian
     inequality of this paper
     proposes a bound on the difference between the generalisation loss and empirical loss. This bound
     holds with high probability, and it depends on the
     number of data points $N$ and on parameter (learning rate $\lambda$). Moreover, for a suitable choice of $\lambda$
     the bound converges to zero at the rate  $\mathcal{O}(1/\sqrt{N})$. 
    The latter rate is consistent with most of finite-sample bounds available in the literature for various type of models. 
\par
     In order to consider non i.i.d. data we assumed
     that inputs and outputs are bounded and they are weakly dependent.
     The latter represents a type of mixing condition.
     Moreover, we had to restrict attention to dynamical systems which transform bounded and weakly dependent inputs to outputs having the same properties. 
     To this end, we required that the dynamical systems satisfy the \emph{exponential convergence} \cite{PavlocTac2011} property. The latter concept originates from
     control theory.
\par
\textbf{Motivation.}
PAC and PAC-Bayesian bounds are a major tool for analyzing
learning algorithms. 
Moreover,
by minimizing the error bound, new, theoretically well-founded learning algorithms can be 
formulated. In particular, PAC-Bayesian error bounds turned out to be useful for 
providing non-vacuous error bounds for neural networks \cite{Dziugaite2017}. While there is a wealth of literature on PAC \cite{shalev2014understanding} and PAC-Bayesian \cite{alquier2021userfriendly,guedj2019primer} bounds, for static models, much less is known on  dynamical systems.
  \par
   \textbf{Generalization bounds for RNNs.}
      PAC bounds for RNN were developed in \cite{KOIRAN199863,sontag1998learning,pmlr-v108-chen20d} using VC dimension,
      in  \cite{WeiRNN,Akpinar_Kratzwald_Feuerriegel_2020,pmlr-v161-joukovsky21a, pmlr-v108-chen20d} using Rademacher complexity, and in 
      \cite{pmlr-v80-zhang18g} using PAC-Bayesian bounds approach.
      However, all the cited papers assume noiseless models,  
      a fixed number of time-steps, that the training data are i.i.d sampled time-series, and the signals are bounded.
      In contrast, we consider (1) noisy models, (2) generalisation loss defined on infinite time horizon, (3) only one single time series available for training data. 

Our contribution is consistent with recent results for linear RNNs \cite{emami2021implicit,cohen-karlik2023learning}, on their ability to extrapolate to longer sequences by training on short sequences with stochastic gradient descent. Thus we provide a different perspective, while extending to a more general class of models.

\par    
\textbf{PAC and PAC-Bayesian bounds for autoregressive and linear models.} 
In  \cite{alquier2012pred,alquier2013prediction} auto-regressive models without exogenous inputs were considered, and the variables were either assumed to be bounded or the loss function was assumed to be Lipschitz. In contrast, we consider nonlinear state-space models with inputs.
That is, the learning problem considered in this paper is
different from that of \cite{alquier2012pred,alquier2013prediction}.
However, we were inspired by \cite{alquier2012pred,alquier2013prediction}.
 In \cite{CDC21paper,eringis2023pacbayesian}
PAC-Bayesian bounds for linear state-space models were derived. In contrast, in this paper we derive bounds for non-linear state-space models. Moreover, the bound of \cite{CDC21paper,eringis2021optimal} does not converge to zero for systems with unbounded noise and inputs. 
The result of \cite{eringis2023pacbayesian} for systems with bounded noise and inputs, where the PAC-Bayesian bound converges to zero,  is a special case of  the result of the present paper.

\textbf{Finite-sample bounds for system identification.}
Guarantees for asymptotic convergence of learning algorithms are a classical
topic in system identification \cite{LjungBook}.
Recently, several publications on finite-sample  bounds for learning dynamical systems were derived, without claiming completeness \cite{simchowitz2018learning,simchowitz2019learning,simchowitz2021statistical,oymak2021revisiting,lale2020logarithmic,foster2020learning,NEURIPS2018_d6288499,Pappas1,SarkarRD21}. 
First, all the cited papers propose a bound which is valid only for models generated by a specific learning algorithm. In particular, these bounds do not characterize the generalisation gap for arbitrary models, i.e., they  are not PAC(-Bayesian) bounds. 
Second, many of the cited papers do not derive bounds on the infinite horizon prediction error.
More precisely, \cite{oymak2021revisiting,SarkarRD21,lale2020logarithmic,Pappas1,Simchowitz_Foster_2020} provided error bounds for the difference of the first $T$ Markov-parameters of the estimated and true system for a specific
identification algorithm. However, in order to characterize the infinite horizon prediction error,
we need to take $T=\infty$.
For $T=\infty$ the cited bounds become infinite, i.e., vacuous. 
Error bounds for certain classes of non-linear dynamical systems were also derived in \cite{OzayCDC2022,sattar2022non,blanke2023flex,foster2020learning,mania2022active,CainesSwitched,shi2022finite,Roy_Balasubramanian_Erdogdu_2021,Ziemann_Sandberg_Matni_2022,Ziemann_Tu_2022,Li_Ildiz_Papailiopoulos_Oymak_2023}, but they assume full state observation and they provide an error bound for a specific learning algorithm.
In contrast, we consider models with unobserved (hidden) states.

\par
\textbf{Outline}. We start by defining empirical and generalisation loss in Section \ref{sec:prelim}, where the main contribution is showcased in an informal theorem. Then we define the type of systems we shall be exploring in Section \ref{sec:class}, and discuss several limitations of the results. With the formal definition of the system, we can formally define the important quantities of this paper in Section \ref{sec:ProblemForm}. Finally, we state the main results of the paper in Section \ref{sec:Results}, where the PAC-Bayesian bound on the generalisation gap is stated. At the end, we present a short illustrative numerical example in Section \ref{sec:example}. 



\textbf{Notation}.
Note that unless otherwise defined, this paper will follow the notation defined by \citet{goodfellow2016deep}, i.e.  $\vx$ is a vector, and $\rvx$ is a random vector.
Let $\F$ denote a $\sigma$-algebra on the set $\Omega$ and $\bP$ be a probability measure on $\F$. Unless otherwise stated all probabilistic considerations will be with respect to the probability space $(\Omega,\F,\bP)$, and we let $\E(\rvz)$ denote expectation of the stochastic variable \deividas{$\rvz:\Omega\rightarrow\reals^{n_\rvz}$. We shall denote the realisation of a stochastic variable $\rvz$ as $\rvz(\omega)$, with $\omega\in\Omega$.}
Each euclidean space is associated with the topology generated by the 2-norm $\|\cdot\|_2$, and the Borel $\sigma$-algebra generated by the open sets. The induced matrix 2-norm is also denoted $\|\cdot\|_2$. 
We use $\triangleq$ to denote ''defined by'', and $\aseq$ to denote that the equality holds almost surely with respect to some underlying probability measure. 

\section{Problem Formulation and PAC-Bayes Setting}

\label{sec:prelim}

In this paper we will consider time-series supervised learning problem. The goal will be to optimise a posterior distribution defined over some set of predictors $\mathcal{H}$. To do this we assume we have only one sequence of training data. Let us fix \emph{bounded} stochastic processes $\rvy(t)\in \sY\subset \reals^{n_y},\rvx(t)\in \sX \subset \reals^{n_x}$ that share the time-axis $t\in\sZ$, i.e. $\rvy(t),\rvx(t)$ are random vectors on $(\Omega,\F,\bP)$. The goal of each predictor $h\in\mathcal{H}$ is to estimate $\rvy(t)(\omega)$ based on current and past values of $\rvx(t)(\omega)$ up to initial time.
Formally, we can think of $h$ as a 
function $h:\bigcup_{k=1}^{\infty} \sX^k \rightarrow \sY$, such that $\hat{\rvy}(t)(\omega)=h(\{\rvx(s)(\omega)\}_{s=0}^t)$. 
\par
We allow the process $\rvx$ to contain $\rvy$ as a component, i.e., $\rvx=\begin{bmatrix} \bar{\rvx}^T & \rvy^T \end{bmatrix}^T$. 
In this case, the predictor uses past values of $\rvy$ to predict the current one, i.e., the predictor is \emph{autoregressive}.
In particular, in this case,  
for the learning problem to be meaningful, 
the class of predictors $\mathcal{H}$ should be such that $h(\{\begin{bmatrix} \bar{\rvx}^T(s) & \rvy^T(s) \end{bmatrix}^T\}_{s=0}^t)$ does not  depend on the value  $\rvy(t)$ of $\rvy$ at time instant $t$. 
\par
Now given a data set $\train=\{\rvy(s)(\omega),\rvx(s)(\omega)\}_{s=0}^{N-1}$ from a single sample of random variables $\{\rvy(s),\rvx(s)\}_{s=0}^{N-1}$. We are interested in
\emph{empirical loss} 
\begin{equation}
    \label{prelim:losses}
    \hat{\mathcal{L}}_N(h)\triangleq\frac{1}{N}\sum_{t=0}^{N-1}\ell(\vy(t),h(\{\rvx(s)\}_{s=0}^t))
\end{equation}  
and \emph{generalisation loss }
\begin{equation}
    \mathcal{L}(h) \triangleq\lim_{t\to\infty} \E[\ell(\rvy(t),h(\{\rvx(s)\}_{s=0}^t))]
\end{equation}
for some loss function $\ell: \sY \times \sY \to \sR_+$. Note that, classically generalisation loss is defined simply as $\E[\ell(\rvy(t),h(\{\rvx(s)\}_{s=0}^t))]$, however in the time-series setting, this would depend on the time $t$, and as such would not give the desired intuition. 
Indeed, in time series prediction, the predictions are updated as new data points become available, so the relevant metric of generalisation power is the prediction 
error as $t \rightarrow +\infty$.

For now, we assume that the limit 
defining the generalisation loss $\mathcal{L}(h)$ exists for any $h \in \mathcal{H}$. Later we will define the class $\mathcal{H}$ in such a manner that 
this  assumption holds. 

With this in mind, the goal of the PAC-Bayesian framework is to analyse the generalisation gap $\Delta_N(h)\triangleq d(E_{h\sim \rho}\mathcal{L}(h),E_{h\sim \rho} \hat{\mathcal{L}}_N(h))$, where $d$ is any convex function and, $E_{h\sim \rho} \mathcal{L} (h)$ and $E_{h\sim \rho} \hat{\mathcal{L}}_N(h)$ denote the expectation of $\mathcal{L} (h)$ and $\hat{\mathcal{L}}_N(h)$, when $h$ is distributed according to $\rho$. 
\\
In this paper we look at the special case of $\Delta_N(h)\triangleq E_{h\sim \rho} \mathcal{L}(h)-E_{h\sim \rho}\hat{\mathcal{L}}_N(h)$, since bounding $\Delta_N(h)$ will provide immediate bounds on $E_{h\sim \rho}\mathcal{L}(h)$. Furthermore we consider predictors $h$ realised by dynamical systems: 
\begin{equation}\label{eq:dyn1}
\begin{split}
    \rvs(t+1)&=f(\rvs(t),\rvx(t)),\quad \rvs(0)=\rvs_0,\\
    h(\{\rvx(s)\}_{s=0}^t)&=g(\rvs(t),\rvx(t)),
\end{split}
\end{equation}
where $\rvs(t)\in\reals^{n_\rvs}$ is the hidden state.
Note that recurrent neural networks (RNNs)
represent a special case of the dynamical systems of the form of \eqref{eq:dyn1}. 
Under suitable assumptions, which will be discussed in Section \ref{sec:ProblemForm}, the generalisation gap is bounded for non i.i.d. data, and predictors realised by dynamical systems. That is, in this paper we shall prove a \emph{Catoni-like PAC-Bayesian inequality}.
Informally, let $\mathcal{H}$ be a family
of predictors which can be realized by a 
dynamical system of the form \eqref{eq:dyn1}. 
\begin{Theorem}[Informal theorem]\label{thm:informal} Let 
$\pi$ be  a probability
density on $\mathcal{H}$ and 
let $\mathcal{M}_{\pi}$ denote the set of all
probability densities for which the corresponding probability measures are absolutely
continuous w.r.t. to the probability measure defined by $\pi$. 
There exist constants $G_1$ and $G_2$, which depend on 
the class of predictors $\mathcal{H}$,
such that for all $\lambda > 0$
and for any $\delta \in (0,1/2)$, the following holds with probability at least $1-2\delta$
%
    \begin{multline} \label{eq:pac:gen}
        \forall\hat{\rho}\in\mathcal{M}_\pi: E_{h\sim \hat{\rho}} \mathcal{L} (h) \le   E_{h\sim \hat{\rho}} \hat{\mathcal{L}}_N(h) \\
           +\dfrac{1}{\lambda}\!\left[ \KL(\hat{\rho} \|\pi) +
	\ln\dfrac{1}{\delta}+ 
  \frac{\lambda^2}{N}  G_1 +\frac{\lambda}{N}G_2  \right]
    \end{multline}
where $\KL(\hat{\rho}\| \pi)\triangleq \E_{h\sim\hat{\rho}} \ln \frac{\hat{\rho}(h)}{\pi(h)}$
denotes the KL-divergence. 
\end{Theorem}
Note that we refer to $\pi$ as a prior distribution and $\hat{\rho}$ as any candidate to a posterior distribution on the space of predictors, however there is no relation to these distributions except for the condition that $\hat{\rho}$ needs to be absolutely continuous w.r.t. $\pi$. The bound in \eqref{eq:pac:gen} suggest that for learning we could choose a posterior $\hat{\rho}$ which minimizes the right-hand side of \eqref{eq:pac:gen}. Then we can either randomly sample a model from that posterior or choose the model with the highest likelihood \cite{alquier2021userfriendly}.
The posterior which minimizes the right-hand side of \eqref{eq:pac:gen}
is known as the \emph{Gibbs posterior} \cite{alquier-15}, and if we identify $\pi$ with its density function, then the density function of the Gibbs posterior is defined by
$\rho_N(h) \triangleq Z^{-1}\pi(h)e^{-\lambda\hat{\mathcal{L}}_N(h)}$, $Z=E_{\theta\sim \hat{\rho}}[e^{-\lambda\hat{\mathcal{L}}_N(h)}]$.
In particular, the model which maximises $\rho_N(h)$ is the one which minimizes the regularised empirical loss $\hat{\mathcal{L}}_N(h)-\frac{1}{\lambda}\ln(\pi(h))$.
This allows us to use the derived bound for deriving learning algorithms, similarly to Catoni-like bounds, and interpret the prior $\pi$ as a regularisation term added to the empirical loss functions.
\begin{Remark}[Asymptotic properties: $O(1/\sqrt{N})$ bound]\label{rem:asymptoticProp}
 The derived bound has similar asymptotic properties to classical Catoni-like bounds. In particular, if we choose $\lambda$ to be of order $\mathcal{O}\left (\sqrt{N} \right )$, then 
the bound 
$r_N(\lambda,\delta,\hat{\rho})\triangleq \dfrac{1}{\lambda}\!\left[ \KL(\hat{\rho} \|\pi) +
	\ln\dfrac{1}{\delta}+ 
  \frac{\lambda^2}{N}  G_1 +\frac{\lambda}{N}G_2  \right]$
converges to zero with the rate $O(1/\sqrt{N})$ for each fixed $\delta$ and posterior $\hat{\rho}$. This implies existence of $N^*>0$, where the proposed bound is non-vacuous for all $N>N^*$. 
\end{Remark}
\begin{Remark}[Intuition behind the constants]
Intuitively, the bounds $G_1$ and $G_2$ are increasing functions of $\theta$-mixing coefficients of the data, and certain  quantities which are related to robustness  of the predictors. 
In particular, the more independent the data is, and the  more robust the predictors are , the smaller  the bound is. 
\end{Remark}

\paragraph{Proof strategy}
The results, relies on PAC-Bayes inequalities obtained by applying Donsker-Varadhan change of measure , in the form of the following lemma. 
\begin{lemma}[Theorem 3 of \citet{nips-16}]
\label{lemma:PAC-BayesianKL_General} For any measurable functions $X(h)=X(h,\omega),Y(h)=Y(h,\omega)$ on $\mathcal{H}$, let $\pi(h)$ be a probability distribution on $\mathcal{H}$, and let $\mathcal{M}_\pi$ denote the set of probability distributions absolutely continuous w.r.t. $\pi$, any $\delta\in(0,1]$, and $\lambda>0$ the following holds with probability at least $1-\delta$
    \begin{multline}
         \forall\hat{\rho}\in\mathcal{M}_\pi:\; E_{h\sim \hat{\rho}}  (X (h)-Y(h)) \\
         \le \dfrac{1}{\lambda}\!\left[\KL(\hat{\rho} \|\pi) + \ln\dfrac{1}{\delta}	+ \Psi_{\pi}(\lambda,N) \right ],
         \label{lemma:PAC-BayesianKL_General:eq1}
    \end{multline}
      with $\Psi_{\pi}(\lambda,N)=\ln \E_{h\sim\pi} \E[e^{\lambda(X(h)-Y(h))}]$.
\end{lemma}


We first apply Lemma \ref{lemma:PAC-BayesianKL_General} with $Y(h,\omega)=\hat{\mathcal{L}}_N(h)$ the empirical loss and \[X(h,\omega)=V_N(h)\triangleq \lim_{\bar{s}\to-\infty}\frac{1}{N}\sum_{t=0}^{N-1}\ell(\rvy(t),h(\{\rvx(s)\}_{s=\bar{s}}^t))\] the Infinite Horizon loss, a version of empirical loss without the transient caused by initial state. Then bound the moment generating function $\E[e^{\lambda(\hat{\mathcal{L}}_N(h)-V_N(h))}]$ by assuming the predictors exhibit exponential convergence property of dynamical systems (see Def. \ref{def:classDef}) together with proposition 4.2 of \citet{alquier2012pred} which uses properties of weakly dependant processes to bound the state process. This way we obtain a first PAC-Bayesian inequality in the form of \eqref{eq:pac:gen}, that covers the issue of choosing the initial state of the predictor.

Next, we apply Lemma \ref{lemma:PAC-BayesianKL_General} with $Y(h,\omega)=V_N(h)$ and $X(h)=\mathcal{L}(h)$ the generalisation loss. Since $\E[V_N(h)]=\mathcal{L}(h)$, we can bound the moment generating function $\E[e^{\lambda(\mathcal{L}(h)-V_N(h))}]$ by the use of a generalisation of Hoefding's lemma \cite[Theorem 6.6]{alquier2012pred}, \cite{rio2000inegalites} .  
We can apply this lemma since combining a data generating system as in Assumption \ref{as:generator} and a predictor, yields a special class of system (see Def. \ref{def:classDef}). This system is driven by an i.i.d. process, and its output is the true and predicted labels, see Lemma \ref{lemma:full_gen_sys}. Hence, the labels and predictions are weakly dependent.

Now we can apply union bound on the two obtained PAC-Bayesian inequalities, resulting in the bound depicted in Theorem \ref{thm:informal}. In fact, we shall use the above proof strategy to prove a more general theorem, whose corollary proves Theorem \ref{thm:informal}. 
For full details see the full proof in Appendix \ref{sec:proofsResults}.



\section{Class of Dynamical Systems}\label{sec:class}
In this paper we will consider predictors realised by a special class of dynamical systems. Furthermore, later we assume that the data is an output of such a system. In this section, we
define this class of systems.
\deividas{To this end, let us consider a generalisation of the dynamical system in \eqref{eq:dyn1}:}
\begin{subequations}\label{eq:generalNLsys}
\begin{empheq}[left=S \empheqlbrace]{align}
    \vs(t+1)&=f(\vs(t),\vv(t)),\label{eq:generalNLsys_f}\\
    \vy(t)&=g(\vs(t),\vv(t)), \label{eq:generalNLsys_h}
\end{empheq}
\end{subequations}
where $\vv: \mathbb{T} \rightarrow \sV$ is an arbitrary input trajectory \deividas{(e.g., an $\rvx(t)$ as in \eqref{eq:dyn1})},  $\vs: \mathbb{T} \rightarrow \sS$ is the state trajectory,  
 $\vy: \mathbb{T} \rightarrow \sY$ is the output  trajectory, and $\mathbb{T} \subseteq \sZ$ is an interval in $\sZ$
 which has no upper bound.
 That is,  either $\mathbb{T}$ is the whole $\sZ$ or
 $\mathbb{T}$ is an interval
 $[t_0, +\infty) \cap \mathbb{Z}$
 for some $t_0$.
 Moreover, the sets 
 $\sS\subset\sR^n$, $\sV\subset\sR^m$, and $\sY\subset\sR^p$ are bounded.
If $\mathbb{T}=[t_0, +\infty) \cap \mathbb{Z}$, then 
$\vs$ is uniquely determined by the initial state 
$\vs(t_0)=\vs_0$ and the input $\vv$, and 
we write $\vs(t)=\vs_{S}(\vs_0,t_0, \vv;t)$ to emphasise the dependence on initial conditions $\vs(t_0)=\vs_0$, and input $\vv$.
Moreover, if $\vv$ is defined on the whole time axis $\mathbb{Z}$,
then $\vs_{S}(\vs_0,t_0, \vv;t)$
will be understood as the solution 
corresponding to the restriction of 
$\vv$ to $\mathbb{T}$.
When $S$ is clear from the context, 
we use $\vs(\vs_0,t_0, \vv;t)$ instead
of $\vs_{S}(\vs_0,t_0, \vv;t)$.
\par
Likewise, we will use 
$\vy_{S}(\vs_0,t_0,\vv;t)$ to denote
the output $\vy(t)$ which corresponds 
to the state trajectory $\vs(t)=\vs_{S}(\vs_0,t_0, \vv;t)$, i.e., $\vy_S(\vs_0,t_0,\vv;t)=g(\vs_{S}(\vs_0,t_0, \vv;t),\vv(t))$. 
As before, we drop the subscript $S$ 
if it  is clear from the context.
\par
In addition,  we will \emph{identify systems of the form (\ref{eq:generalNLsys})
with the pair $S=(\sS,\sY,\sV,f,g)$} of functions.
\par

For systems (\ref{eq:generalNLsys}) to be suitable for time-series prediction, they should demonstrate robustness. In other words, they should be capable of withstanding small perturbations in inputs and initial states without leading to significant changes in the state and output trajectories over time. Without this robustness, even minor numerical rounding errors, when implemented, can accumulate over time, resulting in increasingly inaccurate predictions.
Note that this is not much of an issue when 
the prediction is a static function of 
a finite number of past inputs (auto-regression, systems with fully observed state), or the length of the time-series used for prediction is small.
However, it becomes an issue when, as it is often the case in time-series prediction, we use an increasing number of data points in  prediction, as more of them become available. 
Additionally, it is desirable for predictors to preserve stationarity and mixing properties. Predictors are expected to generate one-step ahead approximations of the input process. Therefore, they should be able to preserve these basic properties, if they cannot even preserve such simple properties, then there is little hope for them being accurate. 
Both of these properties can be guaranteed by requiring certain stability properties defined below.
To this end, we recall the following
\begin{Definition}[UEC and steady-state state and output trajectories \cite{PavlocTac2011}]
The system $S$ from \eqref{eq:generalNLsys}
is called uniformly exponentially convergent (UEC) with constants $C$ and \mihaly{$\tau \in [0,1)$}, if 
        for each bounded $\vv:\sZ \rightarrow \sV$, there exists a unique bounded state trajectory $\vs=\vs_{S,\vv}: \sZ \rightarrow \sS$ of $S$,
        and for any initial state $\vs_0$
        \begin{equation*}
            \|\vs_S(\vs_0,t_0,\vv;t)-\vs_{S,\vv}(t)\|\leq C\tau^{t-t_0}\|\vs_0-\vs_{S,\vv}(t_0)\|
        \end{equation*}
We refer to $\rvs_{S,\rvv}$     
as the \emph{steady-state state trajectory}
of $S$ associated with $\rvv$.
Likewise, we call
$\rvy_{S,\rvv}:\sZ \ni t \mapsto g(\rvs_{S,\rvv}(t),\rvv(t))$
the \emph{steady-state output trajectory} associated with $\rvv$.
\end{Definition}

UEC guarantees that the state trajectories are robust to perturbations in the initial state as the effect of the initial state decays exponentially fast. In particular, 
it guarantees the existence of
steady-state trajectories, which represent the asymptotic behavior of the system as $t \rightarrow +\infty$, and which depend on the input, but not the initial state. 
It can also be viewed as the state trajectory starting at $t_0=-\infty$, with arbitrary initial state, see Lemma~\ref{lemma:stationary_x_v}. Although 
UEC is stronger than
most of the stability notions used for dynamical systems, many systems used in practice have this property \cite{PavlocTac2011}. 
Next, we define the class of systems which has all the desired properties
mentioned above.
\begin{Definition}[Class \class{} system] \label{def:classDef}
We will say that the system from \eqref{eq:generalNLsys} is a class \class{} system, 
with associated constants $C\geq 1,\tau\in[0,1),L_\rvv>0,L_{g,\vs}>0$, and $L_{g,\rvv}>0$ if the following holds:

\textbf{(UEC)} $S$ is UEC with constants $C$ and $\tau$

\textbf{(Exponential robustness in inputs) }
    For any two bounded input trajectories $\vv_1,\vv_2:\sZ\to \sV$,\begin{equation}
    \label{def:classDef:mem}
        \|\vs_{\vv_1}(t)-\vs_{\vv_2}(t)\|\leq L_\vv\sum_{k=1}^\infty \tau^{k-1}\|\vv_1(t-k)-\vv_2(t-k)\|
    \end{equation}
\textbf{(Lipschitz output)}
    The output function $g$ 
    has Lipschitz  constants $L_{g,\vs},L_{g,\vv}>0$, i.e.,
     \(   \|g(\xi_1,v_1)-g(\xi_2,v_2)\|\leq L_{g,\vs}\|\xi_1-\xi_2\|+L_{g,\vv}\|v_1-v_2\| \), 
     for all $\xi_1,\xi_2 \in \sS$, $v_1,v_2 \in \sV$.
\end{Definition}

UEC implies robustness with respect to perturbations in the inital state. 
Exponential robustness in inputs implies that state is robust w.r.t. perturbations in the inputs, and in
fact the effect of  an instantaneous perturbations in the input decays exponentially fast.
The requirement that the output function is Lipschitz ensures that the
output trajectories inherit the favorable robustness properties of the state trajectories.

\begin{Remark}[Role of constants in robustness]
\label{rem:constants}
    The smaller the constants $C$, $\tau$, $L_{\vv}$  and $L_{q,\vs},L_{g,\vv}$ are, the more robust the system is to perturbations in inputs and states. 
    The constants $\tau$ determines how fast the effect of perturbations will decay with time, and it is the most significant measure of robustness. 
\end{Remark}

\par
In addition, it can be shown that \class{} systems transform inputs to outputs in such a way that stationarity (see Def. \ref{def:stationarity}) and mixing are preserved,
see Lemma \ref{lemma:stationary_x_v}.
%
Some remarks are in order. 

\textbf{Contractive systems are of class \class{}. }
According to Lemma \ref{lemma:classExamples}, a  simple sufficient condition for belonging to class \class{} is that $f,g$  are Lipschitz, and $f$ is a contraction in its first argument\deividas{, i.e. $\|f(s,v)-f(s',v)\|_2<\tau \|s-s'\|_2$, with $\tau<1$.}
\par
\textbf{Examples of RNNs which are of class \class{}}
In particular, for RNNs, this boils down to the activation functions to be Lipschitz and to the condition $\text{Lip}(\sigma_f)\|A\|_2<1$, see Table \ref{tab:exampleSys} 
for the corresponding constants. 
Note that commonly used activation functions (ReLu, tanh, sigmoid, linear, etc.,) are  Lipschitz.
Table \ref{tab:exampleSys} shows explicitly how to compute the necessary constants for RNNs.
\par
A sufficient condition for
$\text{Lip}(\sigma_f)\|A\|_2<1$ to hold 
is that the absolute values of all the entries of $A$
are smaller than $(\text{Lip}(\sigma_f)n)^{-1}$.
Note that many trained and used models have small weights, e.g., \cite{Woo_Kim_Jeong_Ko_Lee_2021}. Furthermore, weight regularization, which tends to lower the norm of the weights, is commonly 
used in learning.
\par

\textbf{Systems of class \class{} which are not contractions}
A sufficient condition for a system to be of class \class{} is that $f$ and $g$ are Lipschitz, and that there exist a quadratic Lyapunov function, see conditions 1 and 2 of Lemma \ref{lemma:classExamples}. 
The latter may hold even if $f$ is not a contraction.
\par
A large class of such examples are 
stable linear state-space systems.
In this case, 
$f(\rvs,\rvx)=A\rvs+B\rvx$ and
$g(\rvs,\rvx)=C\rvs+D\rvx$
for suitable matrices $A,B,C,D$, and 
such systems are of class \class{}, if
$A$ is Schur, i.e., its spectral
radius $\rho(A)$  is smaller than $1$. It is well-known in control theory \cite{Katayama:05} that then condition 3 of Lemma \ref{lemma:classExamples} holds. 
However, $f$ is not necessarily a contraction. 
\par
For piecewise-affine functions $f$
(e.g., RNNs with ReLu activation)
the conditions of Lemma \ref{lemma:classExamples}
can be checked 
using Linear Matrix Inequalities (LMI),
see \cite[Theorem 2]{PavlocTac2011}.

\textbf{Interconnection of systems, multi-layer RNNs}
%
%
 %

\begin{table*}
\caption{Example of Class \class{} systems (Lip($\phi$) denotes the Lipschitz constant of $\phi$)}
\label{tab:exampleSys}
\centering
\begin{tabular}{lll}
\toprule
System & Conditions & Constants  \\ \hline
\begin{tabular}[c]{@{}r@{}}$\rvs(t+1)=\sigma_f(A\rvs(t)+B\rvx(t)+b_1)$\\ $\rvy(t)=\sigma_h(C\rvs(t)+D\rvx(t)+b_2)$\end{tabular} & \begin{tabular}[c]{@{}l@{}}$\sigma_f$ $\sigma_h$ are Lipschitz\\  $\text{Lip}(\sigma_f)\|A\|_2 < 1 $\end{tabular} & \begin{tabular}[c]{@{}l@{}}$C=1$, \quad  $\tau=\text{Lip}(\sigma_f)\|A\|_2$\\ $L_\rvv= \text{Lip}(\sigma_f)^{-1}\|B\|_2$, \\ $L_{g,\rvs}=\text{Lip}(\sigma_h)\|C\|_2$\\ $L_{g,\rvv}=\text{Lip}(\sigma_h)\|D\|_2$
\end{tabular}\\ 
\bottomrule
\end{tabular}
\label{tab:my-table}
\end{table*}
In addition, class \class{} property is preserved under interconnection:
series interconnection of two class \class{} systems is also \class{} (see Lemma \ref{lemma:series}). Furthermore, the system constants of combined system can be computed from those of individual system constants. 
In particular, multilayer RNNs which can be represented as a  series interconnection of several single layer RNNs will be of class \class{}, if each individual layer is of class \class{}.


\section{Assumptions on Data and Predictors}\label{sec:ProblemForm} 
Let us now formally state our assumptions on the data and the class of predictors.
\par
In a nutshell, 
we  would like the process $\rvy,\rvx$ to be outputs of a class \class{} system for an i.i.d. input process, and the predictors to be class \class{}
systems.  To this end, we need to consider \class{} systems driven by stochastic inputs.
More precisely, let $S=(\sS,\sY,\sV,f,g)$ be a system of class \class{}, and let $\rvv$ be an essentially bounded stochastic process taking values in $\sV$.
Then for any initial state 
$\vs_0$, time instances $t_0 \le t$, we can define the 
random variables 
$\rvy_S(\vs_0,t_0,\rvv;t)=\rvy_S(\vs_0,t_0,\rvv(\omega);t)$, $\omega \in \Omega$, and
we refer to them 
as the  \emph{output process of $S$  at time $t$ 
for the stochastic input $\rvv$, initial state $\vs_0$ and initial time $t_0$.}
Likewise, 
for every $t \in \sZ$, we define the random variables $\rvy_{S,\rvv}(t)=\rvy_{S,\rvv(\omega)}(t)$, $\omega \in \Omega$ 
and call them
\emph{the steady-state output process 
of $S$ associated with $\rvv$.}
We are now ready to state our assumptions
on the data. 
\begin{Assumption}[Data Generator]
\label{as:generator}
There exists a system $S_g=(\sS_g,\sY \times \sX, \sV_g,f_g,g_g)$  of class \class{} with constants $C_g, \tau_g,L_{g,\vv},L_{g_g,\rvs},L_{g_g,\rvv}$, 
and an essentially bounded i.i.d. process $\rve_g$
such that 
$\begin{bmatrix} \rvy^T,\rvx^T \end{bmatrix}^T$ is the steady-state output process
of $S_g$ associated with
$\rve_g$, i.e., $\begin{bmatrix} \rvy^T,\rvx^T \end{bmatrix}^T=\rvy_{S_g,\rve_g}$.
\end{Assumption}
Note that, we assume the existence of the generator $S_g$, but we do not assume the knowledge of $S_g$.
\par
Assumption \ref{as:generator} can be viewed as a realizability assumption: we assume that the data generator is of the same type as the predictor.
Another reason for Assumption of \ref{as:generator} is that it guarantees certain weak dependence properties of the data, which, in turn, allow us to use extensions of Hoeffding inequalities \cite{rio2000inegalites} to prove PAC-Bayesian bounds.
\par
More precisely, recall from \cite[Definition 5]{alquier2013prediction} the notion of the mixing coefficient $\theta_{\infty,N}^{\rvq}(1)$, $\forall\,1 \le N \in \sZ$, 
of a process $\rvq$.
Following \cite{alquier2013prediction}
we will say that $\rvq$ is \emph{weakly dependent with constants} $B^{\rvq}$ and $\bar{\theta}_{\infty}^{\rvq}$, if 
$\rvq$ is stationary, essentially bounded
and for all $t \in \mathbb{Z}$, 
$B^{\rvq} \ge \|\rvq(t)\|_\infty$ w.p. $1$,
and
$\bar{\theta}_{\infty}^{\rvq}(1) \ge \theta_{\infty,N}^{\rvq}(1)$
for all $N \ge 1$.  
\begin{lemma}
\label{as:gen_mixing:lem}
 Under Assumption \ref{as:generator}, $\rvq=[\rvy, \rvx]$ is 
 weakly dependent with the constants 
\[ B^{\rvq} \triangleq 2\|\rve_g(0)\|_\infty \left ( L_{g_g,\rvv}+\frac{L_{g,\rvv}L_{g_g,\rvs}}{1-\tau_g}\right ),\] 
\[\bar{\theta}_{\infty}^{\rvq}(1) \triangleq 2\|\rve_g(0)\|_\infty \frac{L_{g,\rvv}L_{g_g,\rvs}}{(1-\tau_g)^2}
\]
    
\end{lemma}
Later on we shall see that the bound on the generalisation gap depends on the constants $B^{\rvq}$ and $\bar{\theta}_{\infty}^{\rvq}(1)$.
Hence, we will assume that they are known.
Intuitively, $B^\rvq$ is just an upper bound on the norm of the data. The constant 
$\bar{\theta}_{\infty,N}^{\rvq}(1)$ encodes the information on how non i.i.d. the data is.
In particular, it is zero for i.i.d. data.
The knowledge of these constants is often assumed in  the literature, e.g, \cite{alquier2013prediction, alquier2021userfriendly}.
\par
Similarly, we assume that the predictors $\hat{\rvy}(t)=h(\{\rvx(s)\}_{s=0}^{t})$, take the form of class \class{} system.
\begin{Assumption}[One-Step Predictor]\label{as:predictor}
We will consider predictors  $h$ such that 
there exists a \class{} system
$S_h=(\sS_h,\sY,\sX,\hat{f},\hat{g})$ and an initial state $\rvs_0$ of $S_h$ such  for all
$t \ge 0$, $t \in \sZ$,
$h(\{\rvx(s)\}_{s=0}^{t})$  is the output of $S_h$ at time $t$ for input $\rvx$, initial state $\rvs_0$ and initial time $0$, i.e.,
\( h(\{\rvx(s)\}_{s=0}^{t})=\rvy_{S_h}(\rvs_0,0,\rvx;t) \). 
\end{Assumption}
Note that the \class{} system representation $S_h$ of a predictor $h$ is not unique.

This class of predictors includes RNNs, 
most of standard classes of state-space representations used for time-series prediction and filtering. 
In Section \ref{sec:prelim}, we have informally taken expectation over predictors, however to do that we have to define a parameterisation over predictors.
\begin{Assumption}[Parameterisation]\label{as:parameterisation}
Assume that there exists a compact set $\Theta\subseteq \reals^{n_\theta}$   and a bounded set $\sS \subseteq \mathbb{R}^n$ 
\mihaly{and a mapping
which maps every $\theta$ to a
system of class \class{}
$S_{\theta} \triangleq (\sS,\sY,\sX,\hat{f}_{\theta},\hat{g}_{\theta})$
such that the following holds. 
For any $h \in \mathcal{H}$ there exists $\theta \in \Theta$ such that the system of class \class{} from
Assumption \ref{as:predictor} is of the form $S_h=S_{\theta}$.}
Moreover, we assume that
the functions
 $(\vs,\vv,\theta) \mapsto \hat{f}_\theta(\vs,\vv)$,$(\vs,\vv,\theta) \mapsto \hat{g}=\hat{g}_\theta(\vs,\vv)$, $\theta \mapsto \hat{\vs}_{s,\theta}$ are continuous.
 \end{Assumption}
For every $\theta \in \Theta$
 which corresponds to $h \in \mathcal{H}$ (i.e., $S_h=S_\theta$), we define the following notation:
\begin{equation}
\label{as:parametrisation:eq1}
\begin{split}
 \hat{\rvy}_{\theta}(t|0) &\triangleq \rvy_{S_{\theta}}(\hat{\rvs}_{s,\theta},0,\rvx;t)=h(\{\rvx(s)\}_{s=0}^t), \\ 
\hat{\rvy}_{\theta}(t)&=\rvy_{S_\theta,\rvx}(t)
\end{split}
\end{equation}
i.e., $\hat{\rvy}_{\theta}(t|0)$ is just the prediction at time $t$ generated by $h$, and 
$\hat{\rvy}_{\theta}(t)$ is the output process
$\rvy_{S_\theta,\rvx}(t)$ of the \class{} system  $S_\theta$ associated  with $\rvx$.
\par
From Lemma \ref{lemma:stationary_x_v}
it follows that limit
\[ \lim_{t\to\infty} \E[\ell(\rvy(t),h(\{\rvx(s)\}_{s=0}^{t}))=\lim_{t \to \infty} \E[\ell(\rvy(t),\hat{\rvy}_{\theta} (t\mid 0)) 
\] exists and equals
\( \E[\ell(\rvy(t),\hat{\rvy}_{\theta}(t))] \),
hence the generalisation loss is well-defined. 
 That is, by using the (non unique) identification of $\mathcal{H}$ with $\Theta$ and the notation in
 \eqref{as:parametrisation:eq1},
 the empirical loss and the generalisation error
 for the predictor $\theta \in \Theta$ is as follows:
 \begin{align*}  
 \hat{\mathcal{L}}_N(\theta)&\triangleq\frac{1}{N}\sum_{t=0}^{N-1}\ell(\vy(t),\hat{\rvy}_{\theta}(t\mid 0) ),\\
 \mathcal{L}(\theta)&=\E[\ell(\rvy(t), \hat{\rvy}_{\theta}(t))]. 
 \end{align*}
 Assumption \ref{as:parameterisation} and the identification of predictors with elements of $\Theta$ allow us to define probability densities and the corresponding expectations 
 over predictors as probability densities and expectations over $\Theta$. That is, for a probability density $\rho(\theta)$ and scalar valued function $\phi(\theta)$:
 \begin{align*}
     \underset{\theta\sim\rho}{\E} \phi(\theta) \triangleq \int_{\theta\in\Theta} \rho(\theta)\phi(\theta) d\theta 
 \end{align*}
\begin{Assumption} \label{as:loss} The loss function $\ell$
is Lipschitz, with global Lipschitz constant $L_\ell>0$ 
\end{Assumption}
Note that Assumption \ref{as:parameterisation}
implies that for any predictor 
$\theta \in \Theta$, 
the output $\hat{\rvy}_{\theta}(t|0)$ of the predictor
is bounded, i.e., $\|\hat{\rvy}_{\theta}(t|0)\| \le G_\theta B^{\rvq}$,
where $G_\theta\triangleq \frac{L_{g,\rvs}(\theta)L_\rvv(\theta)}{1-\tau(\theta)}+L_{g,\rvv}(\theta)$.
Since all the processes are bounded, one can use locally Lipschitz loss functions, e.g. square loss, $\ell\left (\hat{\rvy},\rvy \right )=\|\hat{\rvy}-\rvy\|_2^2$, with $L_\ell\geq \sup_{\theta\in\Theta} 2B^\rvq G_\theta \geq \sup_{\theta\in\Theta} 2\|\hat{\rvy}_\theta-\rvy\|_\infty$. Furthermore, while not explicitly stated, we allow for $L_\ell$ to depend on $\theta$, and as such for square error loss one can use $L_\ell(\theta)=2B^\rvq G_\theta$.
\begin{Remark}
    Standard classification does not fit our framework (the output function of the data generator is discontinuous), but classification with soft class labels  \cite{hinton2015distilling} does.  
    To this end, we use the softmax loss function, 
     \(   \ell(\rvy,\hat{\rvy})=-\sum_{i=1}^K y_i\ln \left ( \frac{e^{\hat{\rvy}_i}}{\sum_{j=1}^K e^{\hat{\rvy}_j}} \right )
    \), which by Lemma \ref{lemma:Cross-EntropyLipschitz}, has Lipschitz constant $L_\ell\leq K(2\|\hat{\rvy}(0)\|_\infty+\ln{K}+2)$, which by Lemma \ref{lemma:mixingCoeff} can be upper bounded, i.e., $L_\ell(\theta) \leq K(2B^\rvq G_\theta+\ln{K}+2)$.
\end{Remark}

\section{Results}\label{sec:Results}

With the definitions and assumptions above, we show general bound explicitly. Of which the informal theorem \ref{thm:informal} is a special case.
\begin{Theorem}\label{thm:mainThm} Let predictors $h\in\mathcal{H}$\deividas{, see \eqref{eq:dyn1}}, be realised by \class{} class dynamical systems (see Def. \ref{def:classDef}), and be parameterised by weights $\theta\in\Theta\subset \sR^{n_\theta}$ (see Asm. \ref{as:parameterisation}), with associated system constants $C(\theta),L_{g,\rvs}(\theta),L_{g,\rvv}(\theta),L_{\rvv}(\theta)$\deividas{, and $\tau(\theta)\in[0,1)$}. 
Let $\pi$ be any density on the class of predictors $\mathcal{H}$, and denote by $\mathcal{M}_\pi$ the set of all
probability densities for which the corresponding probability measures are absolutely continuous w.r.t. to the probability measure defined by $\pi$. Under assumptions \ref{as:generator},\ref{as:predictor}, and \ref{as:loss}, for any $\delta\in(0,0.5]$, and $\lambda>0$ the following inequality holds with probability at least $1-2\delta$,
    \begin{multline} \label{eq:mainThmIneq}
         \forall\hat{\rho}\in\mathcal{M}_\pi:\; E_{\theta\sim \hat{\rho}} \mathcal{L} (\theta) \le   E_{\theta\sim \hat{\rho}} \hat{\mathcal{L}}_N(\theta) \\
         +\dfrac{1}{\lambda}\!\left[\KL(\hat{\rho} \|\pi) + \ln\dfrac{1}{\delta}	+ \widehat{\Psi}_{\pi}(\lambda,N) \right ],
    \end{multline}
    with $\widehat{\Psi}_{\pi}(\lambda,N)\triangleq \frac{1}{2}\left ( \ln E_{\theta\sim\pi}\widehat{\Psi}_1(\theta)+\ln E_{\theta\sim\pi}\widehat{\Psi}_2(\theta)  \right )$, 
    \begin{align*}        
    \widehat{\Psi}_1(\theta)&\triangleq \exp\left \{\frac{2\lambda^2 L_\ell^2}{N}\left (B^\rvq \left ( G_\theta+ H_\theta \right )+\bar{\theta}_{\infty}^{\rvq}(1) G_\theta \right )^2 \right \},\\
        \widehat{\Psi}_2(\theta)&\triangleq \exp \left \{\frac{2\lambda L_\ell C(\theta)}{N} \left ( 2 B^\rvq H_\theta+\frac{\|\hat{\vs}_{s,\theta}\| L_{g,\rvs}(\theta) }{1-\tau(\theta)} \right )\right \},\\
        G_\theta &\triangleq \frac{L_{g,\rvs}(\theta)L_\rvv(\theta)}{1-\tau(\theta)}+L_{g,\rvv}(\theta),\; H_\theta\triangleq \frac{L_{g,\rvs}(\theta)L_\rvv(\theta)}{(1-\tau(\theta))^2},
    \end{align*}
    and $\|\hat{\vs}_{s,\theta}\|$ denotes the initial state of the predictor parameterised by $\theta$ and, $\bar{\theta}_{\infty} ^\rvq(1)$ and $B^\rvq$ are properties of data generating system \ref{as:generator} describing the data distribution as in Lemma \ref{as:gen_mixing:lem}.
\end{Theorem}

The quantities $\widehat{\Psi}_1(\theta),\widehat{\Psi}_2(\theta)$ can be computed using Markov-Chain Monte-Carlo methods for standard architectures, e.g. RNNs (see Table \ref{tab:exampleSys}) where system constants have explicit expressions.

\par

Theorem \ref{thm:mainThm}, provides the most general case. Under more strict assumptions, from Theorem \ref{thm:mainThm} we can trivially arrive at Catoni-like bounds:
\begin{corollary}\label{cor:G1G2} If 
\begin{equation}\label{eq:G1G2}
\begin{split}
    G_1 &\triangleq \sup_{\theta\in\Theta}2 L_\ell^2\left (B^\rvq \left ( G_\theta+ H_\theta \right )+\bar{\theta}_{\infty,N}^{\rvq}(1) G_\theta \right )^2,\\
    G_2 &\triangleq \sup_{\theta\in\Theta}2 L_\ell C(\theta) \left ( 2 B^\rvq H_\theta+\|\hat{\vs}_{s,\theta}\|\frac{L_{g,\rvs}(\theta) }{1-\tau(\theta)} \right )
\end{split}
\end{equation}
exist, then Theorem \ref{thm:informal}, with $G_1,G_2$ from \eqref{eq:G1G2} holds.
\end{corollary}
\deividas{
The constants $G_1$ abnd $G_2$ depend on the chosen parametrisation of the hypothesis class.
Note that $\Theta$ might be larger than $\mathcal{H}$, by this we mean that there might be several parameters which correspond to the same hypothesis $h\in\mathcal{H}$. The consequence of this, is that the supremum in \eqref{eq:G1G2} will lead to more conservativeness, note however that this conservativeness decays with length of data, see Remark \ref{rem:asymptoticProp}.}
\begin{figure}[ht]
    \centering
    \includegraphics[width=3in]{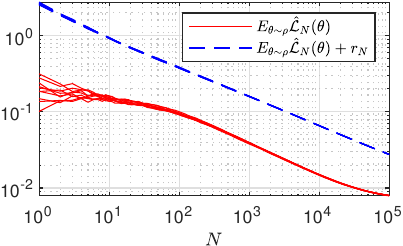} 
    \caption{Theorem \ref{thm:mainThm} is used to compute the results of the example described in Section \ref{sec:example}, evaluated on 10 different realisations of data. }
    \label{fig:example}
\end{figure}
\paragraph*{Discussion on the error bound}
The error bounds is increasing in : 
\par
\textbf{(1)} the magnitude of the data (defined by $B^q$, see Lemma \ref{as:gen_mixing:lem}), 
\par 
\textbf{(2)} in the mixing coefficient $\bar{\theta}_{\infty}^{\rvq}(1)$ of the data generating process, see Lemma \ref{as:gen_mixing:lem}.
The smaller this mixing coefficient is, i.e., the closer the data is to being i.i.d., the smaller is the bound.
In particular, for i.i.d., inputs (i.e., $\bar{\theta}_{\infty}^{\rvq}(1)=0$) we get back the classical PAC-Bayesian bounds \cite{alquier2021userfriendly}.
\par
\textbf{(3)} In the degree of robustness of the predictors
 captured by $\tau(\theta)$,  and to a smaller extent by 
 $L_{g,\rvs}(\theta) ,L_{\vv}(\theta)$, $L_{g,\vv}$, $C(\theta)$.
 By Remark \ref{rem:constants},
 the smaller these constants are, the more robust the predictors are. That is,
 robustness is connected to a smaller generalisation gap. 

\par
\textbf{The role of the number of time steps}
In our setting, due to the definition of the generalisation loss, 
the number of time steps for which the predictor has been run during inference does not enter the bounds. 
The key for achieving this was to assume that the predictors are stable dynamical systems. 
This is in contrast to other bounds for RNNs/dynamical systems \cite{KOIRAN199863,sontag1998learning,pmlr-v108-chen20d,WeiRNN,Akpinar_Kratzwald_Feuerriegel_2020,pmlr-v161-joukovsky21a} which grow with the number of time steps used for inference. The latter makes it difficult to use those bounds to characterize the generalisation loss for inference from long sequences. 
\par
\textbf{Role of the depth of RNNs}
As it was mentioned in
Section \ref{sec:class}
for multi-layer RNNs, 
the constants $\tau(\theta)$, $L_{g,\rvs}(\theta) ,L_{\vv}(\theta)$, $L_{g,\vv}$, $C(\theta)$
can be estimated based on the corresponding constants for each layer. 
However, these estimates will be more conservatives, as the number of layers grow,
making our PAC-Bayesian bound more conservative.
That is, while our bound does not explicitly depend on the depth of the network,  it may  still increase as the depth of the RNNs increases. 
The situation is 
similar to PAC(-Bayesian) bounds for deep neural networks, which also grow with the number of layers. Whether this is a methodological artifact, or it is related to the potential of deep networks to overfit, remains to be explored. 

\section{Illustrative Example} \label{sec:example}

\newcommand{\Ag}{A_g}
\newcommand{\Bg}{B_g}
\newcommand{\bxg}{b_{\rvs,g}}
\newcommand{\Cg}{C_g}
\newcommand{\Dg}{D_g}
\newcommand{\byg}{b_{\rvy,g}}
\newcommand{\ce}{1.27}
In this section we shall explore a synthetic example to illustrate the proposed bound. The code for this example is available in Git repository in \cite{NLPACBayesCode}.
We randomly chose a generator of the form
described in Assumption \ref{as:generator} with:
\begin{subequations}
\begin{align}
\rvs_g(t+1)&=\text{ReLu}\left (\Ag\rvs_g(t)+\Bg\rve_g(t)+\bxg \right),\\
\begin{bmatrix} \rvy(t)\\ \rvx(t)\end{bmatrix}&=\text{tanh}\left (\Cg\rvs_g(t)+\Dg\rve_g(t)+\byg \right ),
\end{align}
\end{subequations}
with $n_\rvs=2,n_\rvy=1,n_\rvx=1$, see numerical values of the weights ($A_g,B_g,\bxg,\Cg,\Dg,\byg$) in \eqref{eq:numexGen} of appendix \ref{sec:exampleAppendix}.
Then Lemma \ref{as:gen_mixing:lem}
holds with
$B^\rvq=\sqrt{2},\;\bar{\theta}_{\infty,N}(1)=2$, and $\|\rve_g(t)\|_\infty\leq \ce$

With the system above we generate data by sampling $\rve_g(t)$ from a truncated Gaussian distribution, and propagating the system above. The predictors are chosen with identical shape to the generator, i.e. using Relu and tanh activation functions, and 2 hidden states, and all weights are parameterised including the initial state. The loss function is square loss. 
We employ Markov Chain Monte-Carlo sampling to compute the various expectations over the prior and posterior appearing in the bound of Theorem \ref{thm:mainThm},
see Appendix \ref{sec:exampleAppendix}.  
\par
The prior is chosen naively as $\pi=\mathcal{N}(0,\sigma^2 I)$, with $\sigma^2=0.02$. 
The posterior is chosen as the Gibbs posterior, i.e. $\hat{\rho}_N(\theta)\propto \pi(\theta)e^{-\lambda_N\hat{\mathcal{L}}(\theta)}$ with $\lambda_N=\sqrt{N}$.

Note that due to using tanh 
in the output equation, with a trivial predictor ( $\hat{\rvy}(t)=0$)
the maximum prediction error is $1$.
As we can see in Figure \ref{fig:example}, the proposed bound is non-vacuous for $N\geq 9$, since $E_{\theta\sim\rho}\mathcal{L}(\theta)\leq E_{\theta\sim\rho}\hat{\mathcal{L}}(\theta)+r_N< 1,\; \forall N\geq 9$. Note that this threshold only holds for this specific set-up and it will change for different problems. However, we claim that, since the bound converges to 0, there exists some $N_0$, s.t. $\forall N>N_0$ the proposed bound is non-vacuous.
\par
Furthermore, our results not only apply to non-linear dynamical systems, but they are much tighter than what was proposed by \citet{eringis2023pacbayesian} for linear dynamical systems, all while we require arguably less information about the generating system.

\section{Conclusion} \label{sec:Conclusion}
In this paper, we have provided non-asymptotic bounds on the generalisation gap of exponentially stable dynamical systems. Under suitable conditions on hyper-parameter $\lambda$, we see that the bound on generalisation gap converges at a rate of $\mathcal{O}(1/\sqrt{N})$. The bound will converge either to 0 for a fixed posterior or to a constant involving only KL divergence, for Gibbs posterior.
Furthermore the bound only depends on quantities related to the magnitude of the label and input processes, and the $\theta_\infty$ mixing coefficient of these processes. Not only does the proposed bound inform us how to design priors that yield smaller generalisation gap, but since the proposed bound only requires limited knowledge of the data generator, we could potentially apply these bounds directly for various applications. 

Potential future research directions include applying our results to various architectures, which would require deriving tools to check if these architectures belong to class \class{}. Furthermore, designing LMIs or other tools to obtain tighter system constants, would immediately yields tighter bounds on the generalisation gap.

\bibliography{bib.bib}

\newpage
\appendix
\onecolumn
\section{Appendix}
\begin{figure}
    \centering
    \includegraphics[width=\linewidth]{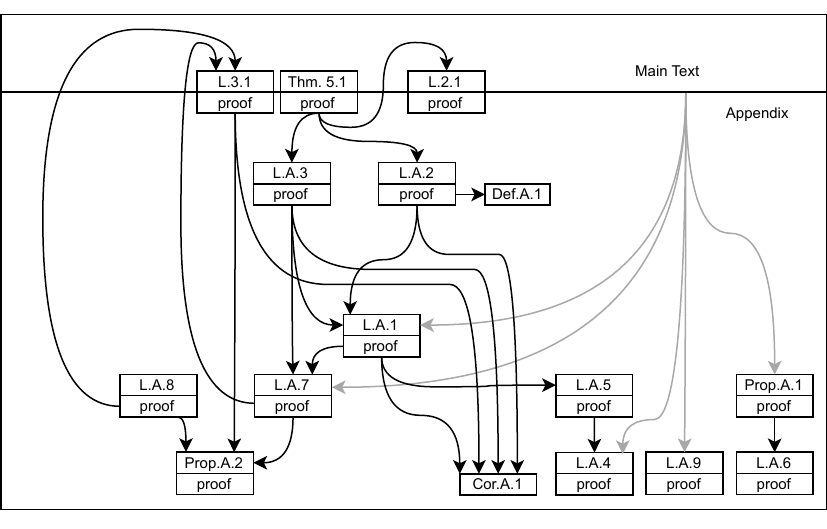}
    \caption{Dependance of Lemmas (and proofs) in the Appendix. For example: Theorem 5.1 directly depends on Lemmas 2.1, A.2, and A.3}
    \label{fig:appendix_dependance}
\end{figure}

To improve readability of the Appendix, we provide Figure \ref{fig:appendix_dependance}, which shows how each Lemma depends on previous or future lemmas. 

By using the identification of $\mathcal{H}$ with $\Theta$ and the notation in
\eqref{as:parametrisation:eq1},
the empirical loss and the generalisation loss
for the predictor $\theta \in \Theta$ can be written follows:
\[   \hat{\mathcal{L}}_N(\theta)\triangleq\frac{1}{N}\sum_{t=0}^{N-1}\ell(\vy(t),\hat{\rvy}_{\theta}(t\mid 0) ), \quad \mathcal{L}(\theta)=\E[\ell(\rvy(t), \hat{\rvy}_{\theta}(t))]. 
\]
\par
Moreover, assumption \ref{as:parameterisation} and the identification of predictors with elements of $\Theta$ allow us to define probability densities and the corresponding expectations 
over predictors as probability densities and expectations over $\Theta$. That is, for a probability density $\rho(\theta)$ and scalar valued function $\phi(\theta)$:
\begin{align}
    \underset{\theta\sim\rho}{\E} \phi(\theta) \triangleq \int_{\theta\in\Theta} \rho(\theta)\phi(\theta) d\theta 
\end{align}

\begin{Definition}[Stationarity]\label{def:stationarity}
    A process $\mathbf{s}$ is stationary 
if the joint distribution of the random variables $(\mathbf{s}(t+k_1),\ldots, \mathbf{s}(t+k_r) )$ does not depend on $t$ for any $k_1,\ldots,k_r$ and $r > 0$.
That is, if you take any significantly large window in your data it must have the same statistical properties as any other window of the same length.
\end{Definition}

\subsection{Lemmas} \label{sec:mainLemmas}
\begin{lemma} \label{lemma:mixingCoeff}
    Let the generator be of class \class{}, and thus WDP (by Lemma \ref{lemma:CBS}, and Corollary \ref{cor:WDP}), and the predictor be of class \class{}, then 
    the mixing coefficient of the process $\rvo(t)=\begin{bmatrix} \rvy^T(t)& \hat{\rvy}^T_{\theta} (t)\end{bmatrix}^T$ (output of full system \ref{lemma:full_gen_sys}) can be upper bounded by the mixing coefficients of the output $\rvq = \begin{bmatrix}\rvy^T(t)& \rvx(t) \end{bmatrix}^T $ of the data generator. 
    \begin{align}           
    \theta_{\infty,N}^{\rvo}(1)&\leq \bar{\theta}_{\infty}^{\rvq} G_\theta + B^\rvq H_\theta\\
    \left \|\begin{bmatrix} \rvy(t)\\ \hat{\rvy}(t)\end{bmatrix}\right \|_\infty &\leq B^\rvq G_\theta
    \end{align}
    where $\bar{\theta}_{\infty}^\rvq(1),B^\rvq$ are as in Lemma \ref{as:gen_mixing:lem}
    \begin{proof}[\textbf{Proof of Lemma \ref{lemma:mixingCoeff}}]
    From Lemma \ref{lemma:CBS}
    $\rvq(t)=\begin{bmatrix} \rvy(t)\\ \rvx(t)\end{bmatrix}$ is a Bernouli shift(CBS), and there exists a function $H_1$ s.t.
    \begin{align}
        \rvq(t)=H_1(\rve_g(t),\rve_g(t-1),\dots)
    \end{align}
    where $\rve(t)$ is a bounded i.i.d. process, furthermore
    \begin{align}
        \|H_1(\rve(t),\rve(t-1),\dots)-H_1(\rve'(t),\rve'(t-1),\dots)\|\leq\sum_{k=1}^\infty a_k(H_1)\|\rve(t-k)-\rve'(t-k)\|\\
        \sum_{k=1}^\infty a_k(H_1)<\infty,\quad \sum_{k=1}^\infty ka_k(H_1)<\infty 
    \end{align}
    
    Now consider $\rvo(t)=\begin{bmatrix}\hat{\rvy}(t)\\ \rvy(t) \end{bmatrix}$.
    Then by Lemma \ref{lemma:stationary_x_v}
    it follows that there exists a function $H$ such that  $\rvo(t)=H(\rvq(t),\rvq(t-1),\dots)$,
    $H$ is Lipschitz, i.e.
    \begin{align}
        \|H(\rvq(t),\rvq(t-1),\dots)-H(\rvq'(t),\rvq'(t-1),\dots)\|\leq \sum_{k=0}^\infty a_k(H)\|\rvq(t-k)-\rvq'(t-k)\|
    \end{align}
    and $\sum_{k=0} a_k(H) <+\infty$ and 
    $\sum_{k=0} k a_k(H) <+\infty$. 
    Moreover,
    \begin{align}
        a_k(H)=\begin{cases}\hat{L}_{g,\rvs}\hat{L}_\rvv \hat{\tau}^{k-1},& k>0 \\ \hat{L}_{g,\rvv},& k=0 \end{cases}
        \label{lemma:mixingCoeff:eq1}
    \end{align}
    where $\hat{\tau}=\tau(\theta)$,
    $\hat{L}_{g,\rvs}=L_{g,\rvs}(\theta)$,
    $\hat{L}_\rvv=L_{\rvv}(\theta)$.
    Then 
    \begin{align}
        \rvo(t)&=H(H_1(\rve(t),\rve(t-1),\dots),H_1(\rve(t-1),\rve(t-2),\dots),\dots)\\
        &\triangleq\bar{H}(\rve(t),\rve(t-1),\dots)
    \end{align}
    and
    \begin{align}
        \|\bar{H}(\rve(t),\rve(t-1),\dots)-&\bar{H}(\rve'(t),\rve'(t-1),\dots)\|\\
        &\leq \sum_{k=1}^\infty a_k(H)\|H(\rvq(t),\rvq(t-1),\dots)-H_1(\rvq'(t),\rvq'(t-1),\dots)\|\\
        &\leq \sum_{k=1}^\infty a_k(H) \sum_{j=1}^\infty a_j(H_1)\|\rve(t-(k+j))-\rve'(t-(k+j))\|\\
        &\leq \sum_{k=1}^\infty\sum_{j=1}^\infty a_j(H_1) a_k(H) \|\rve(t-(k+j))-\rve'(t-(k+j))\|
    \end{align}
   with the transformation $r=k+j \;\Rightarrow\; j=r-k\geq 1$, and thus we claim that $k\leq r-1$, and
    \begin{multline} \label{sumeqsum}
        \sum_{k=1}^\infty\sum_{j=1}^\infty a_j(H_1) a_k(H) \|\rve(t-(k+j))-\rve'(t-(k+j))\| \\
        = \sum_{r=0}^\infty \left ( \sum_{k=1}^{r-1} a_{r-k}(H_1) a_k(H) \right ) \|\rve(t-r)-\rve'(t-r)\|
    \end{multline}

    To prove the claim let us simplify notation by defining 
\begin{align}\label{simp}
    \phi_{k,j}\triangleq a_j(H_1) a_k(H) \|\rve(t-(k+j))-\rve'(t-(k+j))\|\\
    \bar{\phi}_{k,r}\overset{r\triangleq k+j}{=} a_{r-k}(H_1) a_k(H) \|\rve(t-r)-\rve'(t-r)\|.
\end{align}
For fixed $\bar{r}\geq 2$ we have
\begin{align}
    \sum_{k=1}^{\bar{r}-1}\bar{\phi}_{k,\bar{r}}=\sum_{k=1}^{\bar{r}-1} \phi_{k,\bar{r}-k}
\end{align}
then it follows for fixed $l>2$ that
\begin{align}
    \sum_{\bar{r}=2}^l\sum_{k=1}^{\bar{r}-1}\bar{\phi}_{k,\bar{r}} = \sum_{\bar{r}=2}^l \sum_{k=1}^{\bar{r}-1} \phi_{k,\bar{r}-k} = \sum_{k=1}^{l-1} \sum_{j=1}^{l-k} \phi_{k,j}\label{eq:beforeRearranging}
\end{align}
where the second equality is obtained by simple rearrangement of terms. By taking a limit we obtain
\begin{align}
    \sum_{\bar{r}=2}^\infty \sum_{k=1}^{\bar{r}-1}\bar{\phi}_{k,\bar{r}}\triangleq
    \lim_{l\rightarrow \infty}\sum_{\bar{r}=2}^l\sum_{k=1}^{\bar{r}-1}\bar{\phi}_{k,\bar{r}} =  \lim_{l\rightarrow \infty} \sum_{k=1}^{l-1} \sum_{j=1}^{l-k} \phi_{k,j}
     \triangleq \sum_{k=1}^{\infty} \sum_{j=1}^{\infty} \phi_{k,j}
\end{align}
that is
\begin{align}
\sum_{\bar{r}=2}^\infty \sum_{k=1}^{\bar{r}-1}\bar{\phi}_{k,\bar{r}} &= \sum_{k=1}^{\infty} \sum_{j=1}^{\infty} \phi_{k,j}
\end{align}
proving \eqref{sumeqsum}.  As a passing remark we mention that \eqref{sumeqsum} can also be obtained by noting that the map $(k,j)\mapsto(k,r)=(k,k+j)$ is injective and then rearranging terms.

    Taking a step back, we have obtained    
    \begin{multline}
        \|\bar{H}(\rve(t),\rve(t-1),\dots)-\bar{H}(\rve'(t),\rve'(t-1),\dots)\|\\
        \leq \sum_{r=0}^\infty \left ( \sum_{k=1}^{r-1} a_{r-k}(H_1) a_k(H) \right ) \|\rve(t-r)-\rve'(t-r)\|=\sum_{r=0}^{\infty} a_r(\bar{H})  \|\rve(t-r)-\rve'(t-r)\|,
    \end{multline}
    with $a_k(\bar{H})=\sum_{j=1}^{k-1} a_{k-j}(H_1) a_j(H)$.
    By the well-known property of products of absolutely convergent series
    \[ \sum_{k=1}^\infty a_k(\bar{H})=(\sum_{k=1}^{\infty} a_k(H_1)) (\sum_{k=0}^{\infty} a_k(H))   < \infty
    \]
    Moreover, 
    \begin{align}
    %
        & \sum_{i=1}^\infty \sum_{j=1}^{\infty} (i+j)a_{i}(H_1) a_j(H)\\
        &=
        \sum_{i=1}^\infty\sum_{j=1}^{\infty} ( ia_{i}(H_1) a_j(H)+ja_{i}(H_1) a_j(H))\\
        &=
        \left ( \sum_{i=1}^\infty\sum_{j=1}^{\infty}  ia_{i}(H_1) a_j(H)+\sum_{i=1}^\infty\sum_{j=1}^{\infty} ja_{i}(H_1) a_j(H) \right )\\
        &=\left ( 
        \sum_{i=1}^\infty ia_{i}(H_1) \right )\left ( \sum_{j=1}^{\infty}a_j(H) \right )+\left ( 
        \sum_{i=1}^\infty a_i(H_1) \right ) \left (\sum_{j=1}^\infty ja_j(H) \right ) < +\infty
    \end{align}
    As all the elements of the series above are positive,
    it follows that
    \begin{align*}
      & \sum_{k=1}^{\infty} k a_k(\bar{H})=
      \lim_{N \rightarrow +\infty} \sum_{k=1}^{2N} k a_k(\bar{H})=\lim_{N \rightarrow +\infty} 
      \sum_{i=1}^{N} \sum_{j=1}^{N} (i+j) a_{i}(H_1) a_j(H) = \\
       & \sum_{i=1}^\infty \sum_{j=1}^{\infty} (i+j)a_{i}(H_1) a_j(H)  <  +\infty 
    \end{align*}
    
    Now by Proposition 4.2 \citet{alquier2012pred}, we know that
    \begin{align}
        \theta_{\infty,N}^{\rvo}(1)&\leq 2\|\rve\|_\infty \sum_{k=1}^\infty ka_k(\bar{H})=2\|\rve\|_\infty \sum_{k=1}^\infty \sum_{j=1}^{k-1} ka_{k-j}(H_1) a_j(H)
    \end{align}
    now we can apply the inverse of the transformation we applied before
   
    Now, by   Proposition 4.2 from  \citet{alquier2012pred}
     and Lemma \ref{lemma:CBS}, 
    \begin{align}
        \theta_{\infty,N}^{\rvq}(1)\leq 2\|\rve\|_\infty\sum_{i=1}^\infty ia_{i}(H_1) = \frac{2\|\rve\|_\infty L_{g,\rvv}L_{g_g,\rvs}}{(1-\tau_g)^2}\triangleq\bar{\theta}_{\infty}^{\rvq}
        \label{lemma:mixingCoef:eq2.1}
        \\
        \|\rvq(t)\|_\infty \leq 2\|\rve\|_\infty \sum_{k=0}^\infty a_k(H_1)=2\|\rve\|_\infty \left ( L_{g_g,\rvv}+\frac{L_{g,\rvv}L_{g_g,\rvs}}{1-\tau_g}\right )\triangleq B^\rvq
        \label{lemma:mixingCoef:eq2}
    \end{align}
    Thus we have
    \begin{align}
        \theta_{\infty,N}^{\rvo}(1)&\leq \bar{\theta}_{\infty}^{\rvq} \left ( \sum_{j=1}^{\infty}a_j(H) \right ) + B^\rvq \left (\sum_{j=1}^\infty ja_j(H) \right )
    \end{align}
    Now let us focus on $\|\rvo(t)\|_\infty$, by \cite[Proposition 4.2]{alquier2012pred}, we have
    \begin{align}
        \|\rvo(t)\|_\infty &\leq 2\|\rve\|_\infty \sum_{k=1}^\infty a_k(\bar{H})\\
        &=2\|\rve\|_\infty\sum_{k=1}^\infty \sum_{j=1}^{k-1} a_{k-j}(H_1) a_j(H)\\
        &=2\|\rve\|_\infty\sum_{k=1}^\infty \sum_{j=1}^{\infty} a_{k}(H_1) a_j(H)\\
        &=\left (2\|\rve\|_\infty \sum_{j=k}^{\infty} a_{k}(H_1)\right )\left (\sum_{j=1}^{\infty}a_j(H) \right )\\
        &=B^\rvq\left (\sum_{j=1}^{\infty}a_j(H) \right ) 
    \end{align}
    Using \eqref{lemma:mixingCoeff:eq1}
    
    \begin{align}
        \sum_{j=0}^{\infty}a_j(H) = \frac{\hat{L}_{g,\rvs}\hat{L}_\rvv}{1-\hat{\tau}}+\hat{L}_{g,\rvv}\\
        \sum_{j=1}^\infty ja_j(H) = \frac{\hat{L}_{g,\rvs}\hat{L}_\rvv}{(1-\hat{\tau})^2}
    \end{align}
    with this we have the statement of the lemma.
\end{proof}
\end{lemma}

\begin{Definition}\label{def:infLoss} Let $\hat{\rvy}_\theta(t)$ denote the infinite horizon prediction (as defined in \eqref{as:parametrisation:eq1}), and
\begin{equation}
    V_N(\theta)=\frac{1}{N}\sum_{t=0}^{N-1}\ell(\rvy,\hat{\rvy}_\theta(t))
\end{equation}
    denote the infinite horizon loss.
\end{Definition}
The infinite horizon loss can be interpreted as the empirical loss, when the initial state of the predictor is chosen as if the predictor had been running for infinite amount of time. Thus the infinite horizon loss, does not contain prediction error caused by the choice of the initial state of the empirical predictor.

\begin{lemma}\label{lemma:mgf(V-L_hat)}  
Under Assumptions \ref{as:generator},\ref{as:predictor}, with $\lambda>0$ the following holds almost surely
    \begin{align}
        \E[e^{\lambda|V_N(\theta)-\hat{\mathcal{L}}_N(\theta)|}] \leq \exp \left \{ \frac{\lambda L_\ell C(\theta)}{N} \left ( 2 B^\rvq H_\theta+\|\hat{\vs}_{s,\theta}\|\frac{L_{g,\rvs}(\theta) }{1-\tau(\theta)} \right ) \right \}
    \end{align}
\begin{proof}[\textbf{Proof of Lemma \ref{lemma:mgf(V-L_hat)}}] By definitions empirical Loss $\hat{\mathcal{L}}_N(\theta)$ and infinite Horizon loss $V_N(\theta)$ (see Definition \ref{def:infLoss}) we have
        \begin{align}
            |V_N(\theta)-\hat{\mathcal{L}}_N(\theta)|&=\left |\frac{1}{N}\sum_{t=0}^{N-1}\ell(\hat{\rvy}_\theta(t),\rvy(t))-\ell(\hat{\rvy}(t|0),\rvy(t)) \right |\\
            &\leq \frac{1}{N}\sum_{t=0}^{N-1} | \ell(\hat{\rvy}_\theta(t),\rvy(t))-\ell(\hat{\rvy}(t|0),\rvy(t))|
        \end{align}
    Let  $\hat{\tau}=\tau(\theta)$,
    $\hat{L}_{g,\rvs}=L_{g,\rvs}(\theta)$,
    $\hat{L}_\rvv=L_{\rvv}(\theta)$. 
    Recall that $\hat{\rvy}_{\theta}$ is the 
    steady-state ouptut trajectory $\rvy_{\rvx}$ of the system $S(\theta)$ for input $\rvx$, in the sequel we will use $\rvy_{\rvx}$ to denote $\hat{\rvy}_{\theta}$.
    By Assumption \ref{as:loss},  the loss is Lipschitz, and 
    \begin{align}
        |V_N(\theta)-\hat{\mathcal{L}}_N(\theta)|&\leq \frac{1}{N}\sum_{t=0}^{N-1} L_\ell \|\hat{\rvy}_\rvx(t)-\hat{\rvy}(t|0)\|\\
        &\leq \frac{1}{N}\sum_{t=0}^{N-1} L_\ell \hat{L}_{g,\rvs}\|\hat{\rvs}_\rvx(t)-\hat{\rvs}(t|0)\|+L_\ell \hat{L}_{g,\rvv}\|\rvx(t)-\rvx(t)\|
    \end{align}
    where $\rvs_{\rvx}$ is the steady-state state 
    trajectory of $S(\theta)$ for input $\rvx$.
    By exponential stability property of class \class{} systems
    \begin{align}
        |V_N(\theta)-\hat{\mathcal{L}}_N(\theta)|&\leq \frac{L_\ell \hat{L}_{g,\rvs}}{N}\sum_{t=0}^{N-1} \hat{C}\hat{\tau}^t\|\hat{\rvs}_\rvx(0)-\hat{\rvs}_0\|\\
        &\leq  \frac{1-\hat{\tau}^N}{N(1-\hat{\tau})} L_\ell \hat{L}_{g,\rvs} \hat{C} ( \|\hat{\rvs}_\rvx(0)\|+\|\hat{\vs}_0\|)
    \end{align}
    By \cite[Proposition 4.2]{alquier2012pred}
    and Lemma \ref{lemma:CBS}, 
    $\|\hat{\rvs}_\rvx(0)\|\leq 2 \|\rvx_\rvv(0)\|_\infty \frac{\hat{L}_\rvv}{1-\hat{\tau}} $, 
    thus 
    \begin{align}
        |V_N(\theta)-\hat{\mathcal{L}}_N(\theta)|&\leq \frac{L_\ell \hat{L}_{g,\rvs}}{N}\sum_{t=0}^{N-1} \hat{C}\hat{\tau}^t\|\hat{\rvs}_\rvx(0)-\hat{\rvs}_0\|\\
        &\leq  \frac{1-\hat{\tau}^N}{N(1-\hat{\tau})} L_\ell \hat{L}_{g,\rvs} \hat{C} \left ( 2 \|\rvx_\rvv(0)\|_\infty \frac{\hat{L}_\rvv}{1-\hat{\tau}}+\|\hat{\vs}_0\| \right )\\
        &\leq   \frac{L_\ell}{N}   \left ( 2 \|\rvx_\rvv(0)\|_\infty \frac{\hat{L}_\rvv\hat{L}_{g,\rvs} \hat{C}}{(1-\hat{\tau})^2}+\|\hat{\vs}_0\|\frac{\hat{L}_{g,\rvs} \hat{C}}{1-\hat{\tau}} \right )
    \end{align}
    Since $\|\rvx(0)\|_\infty\leq \|[\rvy^T_\rvv(0),\rvx^T_(0)]\|_\infty$, and by Lemma \ref{as:gen_mixing:lem},  $\|[\rvy^T(0),\rvx^T(0)]^T \|_\infty\leq B^\rvq,$ and $H_\theta=\frac{\hat{L}_\rvv\hat{L}_{g,\rvs}}{(1-\hat{\tau})^2}$ we get
    \begin{align}
        |V_N(\theta)-\hat{\mathcal{L}}_N(\theta)| \leq \frac{L_\ell}{N}   \left ( 2 B^\rvq H_\theta+\|\hat{\vs}_0\|\frac{\hat{L}_{g,\rvs} }{1-\hat{\tau}} \right )\hat{C}
    \end{align}
    Since the above holds almost surely, then $\forall \lambda >0$
    \begin{align}
        \E[e^{\lambda|V_N(\theta)-\hat{\mathcal{L}}_N(\theta)|}] \leq \exp \left \{ \lambda \frac{L_\ell}{N}   \left ( 2 B^\rvq H_\theta+\|\hat{\vs}_0\|\frac{\hat{L}_{g,\rvs} }{1-\hat{\tau}} \right )\hat{C} \right \}
    \end{align}
    
    \end{proof}
\end{lemma}

\begin{lemma}\label{lemma:mgf(L-V)} Under assumptions \ref{as:generator},\ref{as:predictor},\ref{as:loss}, with $\lambda>0$, the following holds almost surely
\begin{equation}
    \E\left [e^{\lambda(\E[V_N(\theta)]-V_N(\theta))} \right]\\
    \leq \exp\left \{\lambda^2\frac{L_\ell^2}{2N}\left (B^\rvq \left ( G_\theta+ H_\theta \right )+\bar{\theta}_{\infty}^{\rvq} G_\theta \right )^2 \right \}
\end{equation}
with $\bar{\theta}_{\infty}^\rvq(1),B^\rvq$ are as in Assumption \ref{as:gen_mixing:lem}
\begin{proof}[\textbf{Proof of Lemma \ref{lemma:mgf(L-V)}}] For each predictor parametrized with $\theta$, consider $\rvo(t)=\begin{bmatrix}\hat{\rvy}_{\theta} (t)\\ \rvy(t)\end{bmatrix}$, an output of the full-generator system as defined in Lemma \ref{lemma:full_gen_sys}, then by Lemma \ref{cor:WDP} $\rvo(t)$ is a weakly dependent process.
Now consider
\begin{align}
    H(\rvo(0),\dots,\rvo(N-1))=\frac{1}{L_\ell}\sum_{t=0}^{N-1}\ell(\rvo(t))
\end{align}

then $H$ is $1$-Lipschitz, i.e.
\begin{align}
    |H(\rvo(0),\dots,\rvo(N-1))-H(\rvo'(0),\dots,\rvo'(N-1))|= \frac{1}{L_\ell} \left |\sum_{t=0}^{N-1} \ell(\rvo(t))-\ell(\rvo'(t)) \right |\\
    \leq \frac{1}{L_\ell} \sum_{t=0}^{N-1} \left |\ell(\rvo(t))-\ell(\rvo'(t)) \right |
\end{align}
Now by Assumption \ref{as:loss}, loss is $L_\ell$-Lipschitz, thus $H$ is $1$-Lipschitz 
\begin{align}
        |H(\rvo(0),\dots,\rvo(N-1))-H(\rvo'(0),\dots,\rvo'(N-1))| \leq \sum_{t=0}^{N-1} \|\rvo(t)-\rvo'(t)\|
\end{align}

Now note that, $V_N(\theta)=\frac{1}{N}\sum_{t=0}^{N-1}\ell(\rvo(t))=\frac{L_\ell}{N}H(\rvo(0),\dots,\rvo(N-1))$, and  $\mathcal{L}(\theta)=\E[V_N(\theta)]=\frac{L_\ell}{N}\E[H(\rvo(0),\dots,\rvo(N-1))]$. Now by \cite[Theorem 6.6]{alquier2012pred}:
    \begin{align}
        \E\left [e^{s(\E[V_N(\theta)]-V_N(\theta))} \right]=\E\left [e^{s\frac{L_\ell}{N}(\E[H(\rvo(0),\dots,\rvo(N-1))]-H(\rvo(0),\dots,\rvo(N-1)))} \right]\\
        \leq e^{\left (s\frac{L_\ell}{N}\right )^2N(\|\rvo(0)\|_\infty+\theta_{\infty,n}^\rvo(1))^2/2}\\
        =e^{s^2\frac{L_\ell^2}{N}(\|\rvo(0)\|_\infty+\theta_{\infty,n}^\rvo(1))^2/2}
    \end{align}
By Lemma \ref{lemma:mixingCoeff},
\begin{align}
\theta_{\infty,N}^{\rvo}(1)&\leq \bar{\theta}_{\infty}^{\rvq} \left ( \frac{\hat{L}_{g,\rvs}\hat{L}_\rvv}{1-\hat{\tau}}+\hat{L}_{g,\rvv} \right ) + B^\rvq \frac{\hat{L}_{g,\rvs}\hat{L}_\rvv}{(1-\hat{\tau})^2}\\
    \left \|\rvo(0)\right \|_\infty &\leq B^\rvq \left ( \frac{\hat{L}_{g,\rvs}\hat{L}_\rvv}{1-\hat{\tau}}+\hat{L}_{g,\rvv} \right )
\end{align}
where  $\hat{\tau}=\tau(\theta)$,
    $\hat{L}_{g,\rvs}=L_{g,\rvs}(\theta)$,
    $\hat{L}_\rvv=L_{\rvv}(\theta)$, 
thus 
\begin{multline}
    \E\left [e^{s(\E[V_N(\theta)]-V_N(\theta))} \right]\\
    \leq \exp\left \{\frac{s^2 L_\ell^2}{2N}\left ((B^\rvq+\bar{\theta}_{\infty}^{\rvq}) \left ( \frac{\hat{L}_{g,\rvs}\hat{L}_\rvv}{1-\hat{\tau}}+\hat{L}_{g,\rvv} \right )+ B^q \left ( \frac{\hat{L}_{g,\rvs}\hat{L}_\rvv}{(1-\hat{\tau})^2} \right ) \right )^2 \right \}
\end{multline}
\end{proof}
\end{lemma}

\subsection{Class Properties}\label{ap:classProperties}
Let us take a step back, and analyse the properties of class \class{} systems. Firstly, the unique process (w.r.t. $\rvv$) $\vs_\rvv$ is a stationary process, which allows for analysis when $t\rightarrow\infty$. 
\begin{lemma}\label{lemma:stationary_x_v}\; \\
For any class \class{} system. If $\rvv$ is a bounded, stationary process, then the stead-state state and output trajectory  $\rvs_\rvv$  and 
$\rvy_{\rvv}$ are  bounded and stationary processes.
Moreover, 
 exist functions $H_\rvs:\sV^{\infty} \rightarrow \sS$ and $H_\rvy:s\sV^{\infty} \rightarrow \sY$, where
 $\sV^{\infty}$ is the set of all infinite sequences of elements of $\sV$, s.t., 
    \begin{align}      \rvs_\rvv(t)&\overset{a.s.}{=}H_\rvs(\rvv(t-1),\rvv(t-2),\dots),\\
        \rvy_\rvv(t)&\overset{a.s.}{=}H_\rvy(\rvv(t),\rvv(t-1),\dots),
    \end{align}
    such that 
    for any  two bounded
    sequences $(\vz_1(i))_{i=1}^{\infty}, (\vz_2(i))_{i=1}^{\infty} \in \sV^{\infty}$,
    $i=1,2$
    \begin{align}
     & \| H_{\rvs}(\vz_1(1),\vz_1(2), \dots) - H_{\rvs}(\vz_2(1),\vz_2(2),\dots)\| \nonumber \\ 
        &\leq \sum_{i=1}^\infty a_j(H_{\rvs}) \tau^{i-1} \|\vz_1(i+1)-\vz_2(i+1)\|, \label{lemma:stationaryy_c_v:eq1} \\
        & \| H_{\rvy}(\vz_1(1),\vz_1(2), \dots) - H_{\rvy}(\vz_2(1),\vz_2(2),\dots)\| \nonumber \\
        &\leq 
        a_0(H_{\rvy}) \|\vz_1(1)-\vz_2(1)\|+  \sum_{i=1}^\infty  a_i(H_{\rvy}) 
         \|\vz_1(i+1)-\vz_2(i+1)\|, \label{lemma:stationaryy_c_v:eq2} 
    \end{align}
    where $a_i(H_{\rvs})=L_\rvv \tau^{i-1}$, $a_i(H_{\rvy})=L_{g,\rvs}L_\rvv \tau^{i-1}$
    $i=1,2,\ldots$,  and
    $a_0(H_{\rvy})=L_{g,\rvv}$,
    and 
    \begin{align}
     \sum_{j=0}^{\infty} a_j(H_{\rvs}) < +\infty, \quad
     \sum_{j=0}^{\infty} ja_j(H_{\rvs}) < +\infty, \quad
     \sum_{j=0}^{\infty} a_j(H_{\rvy}) < +\infty,\quad 
      \sum_{j=0}^{\infty} ja_j(H_{\rvy}) < +\infty
      \label{lemma:stationaryy_c_v:eq3} 
    \end{align}
\begin{proof}[\textbf{Proof of Lemma \ref{lemma:stationary_x_v}}]
For any given state $\vs_0\in\sS$, and $0\leq N\in\sZ$, Consider
\begin{equation}
    \rvs_N(t)\triangleq\rvs(\vs_0,t_0=t-N,\rvv,t)
\end{equation}
i.e. $\rvs_N(t)$ is a solution of \eqref{eq:generalNLsys_f} started at $\vs(t_0)=\vs_0$, with $t_0=t-N$.

Existence of process $\rvs_N(t)$:
we can show the existence by induction, i.e.:  $\rvs_0(0)=\vs_0$ is well defined, and $\rvs_N(t)=f(\rvs_{N-1}(t-1), \rvv(t))$, thus for all $t=N$, the process is well-defined. 

From uniform convergence property, 
it follows that 
\begin{align}
    \|\rvs_N(t)-\rvs_\rvv(t)\|&\leq C \tau^{N} \|\vs_0-\rvs_\rvv(t-N)\|\leq C \tau^{N} (\|\vs_0\|+R)\\
    R&=\sup_{t}\|\rvs_\rvv(t-N)\|=\sup_{t}\|\rvs_\rvv(t)\| < \infty
\end{align}
Note that $R\leq \frac{C}{1-\tau}$, by \citet[Theorem 1]{PavlocTac2011}.
Since $\tau<1$, $\|\rvs_N(t)-\rvs_\rvv(t)\|\rightarrow 0$, as $N\rightarrow\infty$ with probability one.

There exists a function $G_N:\sV^N \times \sS \to \sS$, s.t.
\[\lim_{N\rightarrow\infty} \rvs_N(t)\aseq \lim_{N\rightarrow\infty} G_N \left (\{\rvv(t-j)\}_{j=1}^N , \vs_0\right )\]
By assumption on class \class{} system, a map $\rvs_\rvv:\sZ\to\sS$ exist for each $\rvv$.

Now consider any $\normlone$ function $H:\sS^r\to \reals$, and the expectation
\begin{multline*}
    \E[H(\rvs_\rvv(k+t_1),\dots,\rvs_\rvv(k+t_r))]\aseq \lim_{N\rightarrow\infty} \E[H(\rvs_N(k+t_1),\dots,\rvs_N(k+t_r))]\\
    =\lim_{N\rightarrow\infty} \E[H\left (G_N \left (\{\rvv(k+t_1-j)\}_{j=1}^N, \vs_0 \right ),\dots, G_N \left (\{\rvv(k+t_r-j)\}_{j=1}^N, \vs_0 \right ) \right )]
\end{multline*}
For the sake of notation let us define
\begin{multline}
    \hat{G}_N(\{\rvv(k+t_i-j | i=\{1,\dots,r\},j=\{1,\dots,N\} \}) \\ 
    \triangleq H\left (G_N \left (\{\rvv(k+t_1-j)\}_{j=1}^N, \vs_0 \right ),\dots, G_N \left (\{\rvv(k+t_r-j)\}_{j=1}^N, \vs_0 \right ) \right )
\end{multline}
Since, $\rvv(t)$ is stationary, i.e. for any function $g$, we have $\E[g(\rvv(t+k)]=\E[g(\rvv(t))]$, then
\begin{multline}
 \E[\hat{G}_N(\{\rvv(k+t_i-j | i=\{1,\dots,r\},j=\{1,\dots,N\} \})]\\
 =\E[\hat{G}_N(\{\rvv(t_i-j | i=\{1,\dots,r\},j=\{1,\dots,N\} \})].
\end{multline}
With this in mind, we have 
\begin{align}
    \E[H(\rvs_\rvv(k+t_1),\dots,&\rvs_\rvv(k+t_r))]\nonumber \\
    &= \lim_{N\rightarrow\infty } \E[\hat{G}_N(\{\rvv(k+t_i-j | i=\{1,\dots,r\},j=\{1,\dots,N\} \})]\\
    &= \lim_{N\rightarrow\infty }\E[\hat{G}_N(\{\rvv(t_i-j | i=\{1,\dots,r\},j=\{1,\dots,N\} \})]\\
    &= \E[H(\rvs_\rvv(t_1),\dots,\rvs_\rvv(t_r))]
\end{align}
Hence, as $H$ is arbitrary, $\rvs_\rvv$ is stationary.

Finally, for any sequence 
$\vz=(\vz(i))_{i=1} \in \sV^{\infty}$, 
consider the trajectory 
$\vv_{\vz}:\sZ \rightarrow \sV$, 
$\vv_{\vz}(-i)=\vz(i)$ $i=1,2,\ldots,$
It then follows that there exists a unique steady-state
state trajectory $\vs_{\vv}$ and output trajectory 
$\vy_{\vv}$ of $S$. Define
$H_{\rvs}(\vz_1,\vz_2,\dots)=\vs_\vv(0)$ and 
$H_{\rvs}(\vz_1,\vz_2,\dots)=\vy_{\vv}(-1)$. 

From \eqref{def:classDef:mem} and the Lipschitz property of $g$ it follows that \eqref{lemma:stationaryy_c_v:eq1}--\eqref{lemma:stationaryy_c_v:eq2} hold.  

Indeed, under class \class{} system, for some $\tau\in [0,1)$, and $0<L_\rvv<+\infty$:
    \begin{align}
        \|\vs_{\vv_{\vz_1}}(0)-\vs_{\vv_{\vz_2}}(0)\|\leq L_\rvv\sum_{i=1}^\infty \tau^{i-1} \|\vv_{\vz_1}(-i)-\vv_{\vz_2}(-i)\|,
    \end{align}
    and $\forall L_\rvv<\infty,\;\forall \tau\in[0,1):$ $L_\rvv\sum_{i=1}^\infty \tau^{i-1}=\frac{L_\rvv}{1-\tau}< +\infty$, and $L_\rvv\sum_{i=1}^\infty i\tau^{i-1}=\frac{L_\rvv}{(1-\tau)^2}<+\infty$. where the closed-form solution to infinite sums is obtained through sum of geometric series, and sum of arithmetic-geometric series.
    
    Likewise, by Lipschits condition on the output function $g$, by setting $\vv_i=\vv_{\vz_i}$
    \begin{align}
        \|\vy_{\vv_1}(0)-\vy_{\vv_2}(0)\|=\|g(\vs_{\vv_1}(0),\vv_1(0))-g(\vs_{\vv_2}(0),\vv_2(0))\| \\
        \leq L_{g,\rvs} \|\vs_{\vv_1}(0) - \vs_{\vv_2}(0)\| + L_{g,\rvv} \| \vv_1(0)-\vv_2(0)\|\\
        \leq L_{g,\rvs} L_\rvv\sum_{i=1}^\infty \tau^{i-1} \|\vv_1(-i)-\vv_2(-i)\| + L_{g,\rvv} \| \vv_1(0)-\vv_2(0)\| 
    \end{align}
    therefore, $L_{g,\rvv}+ L_{g,\rvs} L_\rvv\sum_{i=1}^\infty \tau^{i-1}=L_{g,\rvv}+\frac{L_{g,\rvs}L_\rvv}{1-\tau}< +\infty$, and $L_{g,\rvs} L_\rvv\sum_{i=1}^\infty i\tau^{i-1}=\frac{L_{g,\rvs} L_\rvv}{(1-\tau)^2}< +\infty$. 
\par
Moreover, from uniqueness of steady-state trajectories it also follows that
$H_{\rvs}(\rvv(t-1),\rvv(t-2),\dots)=\rvs_{\rvv}(t)$ and
$H_{\rvy}(\rvv(t),\rvv(t-1),\dots)=\rvy_{\rvv}(t)$.
Indeed, it is easy to verify that 
$\rvs: t \mapsto  H_{\rvs}(\rvv(t-1),\rvv(t-2),\dots)$
and $\rvy: t \mapsto  H_{\rvy}(\rvv(t),\rvv(t-1),\dots)$
also satisfy \eqref{eq:generalNLsys}. 

To see this, for any function $r$ defined on $\sZ$
define the $\tau$-shifted function $\sigma^{\tau} r$
for any $\tau \in \sZ$ as $\sigma^{\tau} r(s)=r(s+\tau)$ for $s \in \sZ$.  From uniqueness of steady-state state trajectories it follows that $\vs_{\sigma^{\tau} \vv}=\sigma^{\tau} \vs_{\vv}$. 
Moreover, from the infinite memory property it follows
that if $\vv_1(s)=\vv_2(s)$ for all $s < t$, then
$\vs_{\vv_1}(s)=\vs_{\vv_2}(s)$ for all $s \le t$. 
Notice that if $\vz=(\vz(i))_{i=1}^{\infty}$ and
$\vz^{'}=(\vz(i))_{i=2}^{\infty}$, it then follows
that for all $s < 0$, $\sigma^{-1} \vv_{\vz}(s)=\vz(-s-1)=\vv_{\vz^{'}}(s)$. Hence,
$H_{\rvs}(\vz_2,\vz_3,\ldots,)=\vs_{\vv_{\vz^{'}}}(0)=\vs_{\sigma^{-1} \vv_{\vz}}(0)=\sigma^{-1} \vs_{\vv_{\vz}}(0)=\vs_{\vv_{\vz}}(-1)$.
Hence, $H_{\rvs}(\vz_1,\vz_2,\ldots)=\vs_{\vv_\vz}(0)=f(\vs_{\vv_\vz}(-1),\vv_{\vz}(-1))=f(H_{\rvs}(\vz_2,\vz_3,\dots),\vz_1)$. 
Moreover,  $H_{\rvy}(\vz_1,\vz_2,\dots)=g(\vs_{\vv_{\vz}}(-1),\vz_1)=g(H_{\rvs}(\vz_2,\vz_3,\dots),\vz_1)$.
Hence, for $\rvs(t)=H_{\rvs}(\rvv(t-1),\rvv(t-2),\dots)$.
$\rvs(t)=f(\rvs(t-1),\rvv(t-1))$ for all $t \in \sZ$. 
As for $\rvy(t)$, notice that $\rvy(t)=g(H_{\rvs}(\rvv(t-1),\rvv(t-2),\dots),\rvv(t))=g(\rvs(t),\rvv(t))$. 

\end{proof}
\end{lemma}

\begin{lemma}\label{lemma:CBS} For any class \class{} system. If $\rvv=\rve$ is a bounded i.i.d. process, then the bounded solution $\rvs_\rve(t)$, and the corresponding output $\rvy(t)=g(\rvs_\rve(t),\rve(t))$ are bounded Causal Bernoulli shifts (CBS) as defined in \cite{alquier2012pred}, i.e., there exist functions $H_\rvs$ and $H_\rvy$ s.t., 
    \begin{align}      \rvs_\rve(t)&\overset{a.s.}{=}H_\rvs(\rve(t-1),\rve(t-2),\dots),\\
        \rvy_\rve(t)&\overset{a.s.}{=}H_\rvy(\rve(t),\rve(t-1),\dots),
    \end{align}
    and $H_{\rvs},H_{\rvy}$ satisfy \eqref{lemma:stationaryy_c_v:eq1}, \eqref{lemma:stationaryy_c_v:eq2} and \eqref{lemma:stationaryy_c_v:eq3}.
    \begin{proof}[\textbf{Proof of Lemma \ref{lemma:CBS}}] \; \\
    By Lemma \ref{lemma:stationary_x_v}, $\rvs_\rve(t)$ is stationary, and by uniform exponential convergence property of class \class{} system, $\rvs_\rve(t)$ is bounded.

    Let us consider a function $H_{\rvs},H_{\rvy}$, from Lemma \ref{lemma:stationary_x_v}.
    Then the statement the lemma follows from
    Lemma \ref{lemma:stationary_x_v} and the definition 
    of CBS.
\end{proof}
\end{lemma}
\begin{corollary}[By Proposition 4.2 \citet{alquier2012pred}]\label{cor:WDP}
    For any class \class{} system. If $\rvv=\rve$ is a bounded i.i.d. process, then the unique bounded solution $\rvs_\rve(t)$ and corresponding output $\rvy_\rve(t)$ are weakly dependent processes as defined in \cite{Dedecker_2007,alquier2012pred}.
\end{corollary}
\begin{proof}[\textbf{Proof of Lemma \ref{as:gen_mixing:lem}}]
 From Lemma \ref{lemma:CBS} it follows that 
 $\rvq(t)=\begin{bmatrix} \rvy(t) \\ \rvx(t) \end{bmatrix}$
is CBS and $\rvq(t)=H_{\rvy}(\rve_g(t),\rve_g(t-1),\dots)$ satisfies \eqref{lemma:stationaryy_c_v:eq1}, \eqref{lemma:classExamples:eq2}, \eqref{lemma:stationaryy_c_v:eq3} with
$S=S_g$. 
Hence by [By Proposition 4.2 \citet{alquier2012pred}]
$\rvq(t)$ is bounded, stationary, weakly dependent and $\bar{\theta}_{\infty}^\rvq(1),B^\rvq$ are as in Lemma \ref{lemma:mixingCoeff}
\end{proof}


\begin{lemma}\label{lemma:classExamples} 
Consider a system 
$S=(\sS,\sY,\sU, f,g)$.
Assume that  $g$ is Lipschitz, and one of
the following conditions hold:
\begin{enumerate}
    \item $\exists \tau<1:$ $\|f(\rvs,\rvv)-f(\rvs',\rvv)\|\leq L_1 \|\rvs-\rvs'\|$, i.e. $f$ is a contraction in $\rvs$,
    \item $\exists P=P^T>0, L_2 > 0, 0 < \tau <1$: $\|f(\rvs,\rvv)-f(\rvs',\rvv')\|_P \leq \tau \|\rvs-\rvs'\|_P+L_2\|\rvv-\rvv'\|$,
    where $\|z\|_P=\sqrt{z^TPz}$ for all $z \in \mathbb{R}^n$,
    \item  $f$ is Lipschitz, it is 
    continuously differentiable  in $\rvs$ and $\exists P=P^T>0,~~\mu\in[0,1)$:
    \begin{align}
        \frac{\partial f}{\partial \vs}(\vs,\vv)^TP\frac{\partial f}{\partial \vs}(\vs,\vv) < \mu P,\; \text{for almost all }\vs,\vv
    \end{align}
\end{enumerate}
Then $S$ is of class \class{}.
\end{lemma}
Concerning other examples of class S systems, notice that 
from Lemma \ref{lemma:classExamples} it follows that a sufficient condition for a dynamical system to be in class S is that it is contraction after a linear change of coordinates. If $P > 0$ is the matrix from Lemma \ref{lemma:classExamples}, then the system obtained from the original one by using the linear state-space transformation $S=P^{-1/2}$ will be contraction. That is, class S systems will contain not only contractive dynamical systems, but
the result of applying a linear state-space transformation to contractive dynamical systems.
Note that the latter is no longer necessarily contractive.
The conditions of Lemma \ref{lemma:classExamples} are only sufficient, we refer to \citet{PavlocTac2011} regarding the gap between the sufficiency of these conditions and necessity. This gap is not trivial, and represents a separate research direction.

\begin{proof}
 Since Condition 1 is a special case of Condition 2 with $P$ being the identity matrix, it is enough to
 show that Condition 2. implies that $S$ is class \class{}. 
 To this end, we show that $S$ satisfies all the
 condition of the definition of class \class{}.
 
 By \cite[Theorem 1]{PavlocTac2011} Condition 2.
 implies that $S$ is UEC with some constants $C > 0$
 and $\tau > 0$. 
 Consider two steady-state state trajectories $\vs_{\rvv_1}$
 and $\rvs_{\vv_2}$ of $S$ for two bounded
 inputs $\vv_i:\mathbb{Z} \rightarrow \mathbb{U}$.
 Note that $\rvs_{\rvv_i}$ are bounded. 
 Note that for any $t \in \mathbb{Z}$,
   \begin{multline}
   \|\vs_{\vv_1}(t)-\vs_{\vv_2}(t)\|_{P} =
    \|f(\vs_{\vv_1}(t-1),\vv_1(t))-f(\vs_{\vv_2}(t-1),\vv_2(t))\|_{P} \le \\
    \tau \|\vs_{\vv_1}(t-1)-\rvs_{\vv_2}(t-1)\|_{P}+L_2 \|\vv_1(t)-\vv_2(t)\|    
\end{multline} 
By repeating the argument above, it follows that
for any $t_0 < t$,
\begin{equation}
\label{lemma:classExamples:eq1}
\|\vs_{\vv_1}(t)-\vs_{\vv_2}(t)\|_{P} \le 
    \tau^{t_0} \|\vs_{\vv_1}(t-t_0)-\rvs_{\vv_2}(t-t_0)\|_{P}+
    \sum_{k=1}^{t_0} \tau^{k-1} L_2 \|\vv_1(t-k)-\vv_2(t-k)\|    
\end{equation}
Note that as for some $B > 0$,
$\|\vs_{\vv_i}(t)\| \le B$, it follows that
$\|\vs_{\vv_i}(t)\|_P \le \sqrt{\lambda_{\max}(P)} B$, where $\lambda_{\max}(P)$ is the maximal eigenvalue of $P$. It then follows that
$0 \le \tau^{t_0} \|\vs_{\vv_1}(t-t_0)-\vs_{\vv_2}(t-t_0)\|_{P} \le \tau^{t_0} 2 \sqrt{\lambda_{\max}(P)} B$ and hence 
$\lim_{t_0 \rightarrow +\infty} \tau^{t_0} \|\vs_{\vv_1}(t-t_0)-\vs_{\vv_2}(t-t_0)\|_{P}=0$.
By taking the limit of the right-hand side of
\eqref{lemma:classExamples:eq1} as $t_0 \rightarrow +\infty$ and noticing that
$\|\vs_{\vv_1}(t)-\vs_{\vv_2}(t)\| \le \frac{1}{\sqrt{\lambda_{\min}(P)}} \|\vs_{\vv_1}(t)-\vs_{\vv_2}(t)\|_{P}$, where 
$\lambda_{\min}(P)$ is the smallest eigenvalue of $P$,
it follows that
\eqref{def:classDef:mem} holds with $L_{\vv}=\frac{L_2}{\sqrt{\lambda_{\min}(P)}}$.
Finally, $g$ is Lipschitz by assumption. That is, all the conditions of class \class{} hold.

Next, we show that Condition 3. implies Condition 2.
Note that Condition 3. implies
that  for all $z \in \mathbb{R}^n$
\begin{equation}
\label{lemma:classExamples:eq2}
   \|\frac{\partial f}{\partial \rvs}(\vs,\vv)z\|_P \le \sqrt{\mu} \|z\|_P
\end{equation}
Define
$g(\alpha)=f(\alpha \vs_1 + (1-\alpha) \vs_2,\vv_1)$
for any $\alpha \in [0,1]$.
It then follows that 
$\frac{\partial g(\alpha)}{\partial \alpha}=\frac{\partial f(\alpha \vs_1 + (1-\alpha) \vs_2,\vv_1)}{\partial \vs}(\vs_1-\vs_2)$.
Hence, 
\[ 
    f(\vs_1,\rvv_1)-f(\vs_2,\vv_1) = g(1)-g(0)
    \int_{0}^{1} \dfrac{f(\alpha \vs_1 + (1-\alpha) \rvs_2),\vv_1)}{d\rvs}(\vs_1-\vs_2)d\alpha. 
\]
From \eqref{lemma:classExamples:eq2} it 
then follows that 
\begin{equation}
\label{lemma:classExamples:eq3}
   \begin{split}
   &  \| f(\vs_1,\vv_1)-f(\vs_2,\vv_1) \|_P \le 
\int_{0}^{1} \|\dfrac{f(\alpha \vs_1 + (1-\alpha) \vs_2),\vv_1)}{d\vs}(\vs_1-\vs_2)\|_{P} d\alpha  \le \\
 & \le \int_0^{1} \sqrt{\mu} \|\vs_1-\vs_2\|_P d\alpha = \sqrt{\mu} \|\vs_1-\vs_2\|_P
\end{split}
\end{equation}
Using \eqref{lemma:classExamples:eq3}
\[
    \begin{split}
     &    \| f(\vs_1,\vv_1)-f(\vs_2,\vv_2) \|_P
      \le \| f(\vs_1,\vv_1)-f(\vs_2,\vv_1) \|_P+
      \|f(\vs_2,\vv_1)-f(\vs_2,\vv_2)\|_{P} \le \\
    & \sqrt{\mu}  \|\vs_1-\vs_2\|_P+
      \sqrt{\lambda_{\max}(P)} \|f(\vs_2,\vv_1)-f(\vs_2,\vv_2)\|
    \end{split}
\]
Since $f$ is Lipschitz, $\|f(\vs_2,\vv_1)-f(\vs_2,\vv_2)\| \le L_{f,\vv} \|\vv_1-\vv_2\|$, where
$L_{f,\vv}$ is the Lipschitz constant of $f(\vs,\vv)$ w.r.t. $\vv$. Hence, it follows that 
Condition 2. holds with $L_2=\sqrt{\lambda_{\max}(P)} L_{f,\vv}$ and $\tau=\sqrt{\mu}$.

\end{proof}
\begin{proposition} RNNs of the form  
        \newcommand{\hs}{\hat{\rvs}}
        \begin{subequations}
        \begin{align}
            \hs(t+1)=\sigma_f(A\hs(t)+B\rvx(t)+b_f) \label{eq:rnn_ex:f}\\
            \hat{\rvy}(t)=\sigma_g(C\hs(t)+D\rvx(t)+b_g) \label{eq:rnn_ex:g}
        \end{align}
        \end{subequations}
        with $\sigma_f,\sigma_g$ Lipschitz continuous, with global constants $\rho_f,\rho_g$ respectively, and $\rho_f\|A\|_2<1$ are of Class \class{}, with constants
        \begin{align}
            C=1,\; \tau= \rho_f\|A\|_2,\; L_\rvv=\rho_f \|B\|_2,\; L_{g,\rvs}=\rho_g\|C\|_2,\;L_{g,\rvv}=\rho_g\|D\|_2
        \end{align}
    \begin{proof}
        
        We show uniform exponential convergence via contraction, see \cite[Theorem 1]{PavlocTac2011}.
        First, see that \eqref{eq:rnn_ex:f} is Lipschitz in the state $\hs$, i.e.
        \begin{align}
            \|\sigma_f(A\hs(t)+B\rvx(t)+b_f)-\sigma_f(A\hs'(t)+B\rvx(t)+b_f)\| \leq \rho_f \|A\hs(t)-A\hs'(t) \|\\
            \leq \rho_f \|A\|_2 \|\hs(t)-\hs'(t) \|
        \end{align}
        thus if $\rho_f \|A\|_2 < 1$, then we have uniform exponential convergence, and $\tau=\rho_f \|A\|_2,C=1$.
        
        Next we show the infinite memory property
        \begin{align}
            \|\hs_\rvx(t)-\hs_{\rvx'}(t)\|\leq \rho_f \|A\hs_\rvx(t-1)+B\rvx(t-1)+b_f-A\hs_{\rvx'}(t-1)-B\rvx'(t-1)-b_f \|\\
            \leq \rho_f\|A\|_2 \|\hs_\rvx(t-1)-\hs_{\rvx'}(t-1)\| + \rho_f \|B\|_2 \|\rvx(t-1)-\rvx'(t-1)\|
        \end{align}
        If we continue applying the same steps, e.g.
        \begin{multline}
            \|\hs_\rvx(t)-\hs_{\rvx'}(t)\|\leq (\rho_f\|A\|_2)^2  \|\hs_\rvx(t-2)-\hs_{\rvx'}(t-2)\| \\
            + \rho_f\|A\|_2\rho_f \|B\|_2 \|\rvx(t-2)-\rvx'(t-2)\| + \rho_f \|B\|_2 \|\rvx(t-1)-\rvx'(t-1)\|
        \end{multline}
        we obtain
        \begin{multline}
            \|\hs_\rvx(t)-\hs_{\rvx'}(t)\|\leq \lim_{r\rightarrow\infty} (\rho_f\|A\|_2)^r \|\hs_\rvx(t-1)-\hs_{\rvx'}(t-1)\| \\
            + \sum_{k=1}^r (\rho_f\|A\|_2)^{k-1} \rho_f \|B\|_2 \|\rvx(t-k)-\rvx'(t-k)\|
        \end{multline}
        now under the assumption that $\rho_f \|A\|_2 < 1$, the limit $\lim_{r\rightarrow\infty} (\rho_f\|A\|_2)^r \|\hs_\rvx(t-1)-\hs_{\rvx'}(t-1)\| = 0$, and thus we have shown the infinite memory property with $\tau=\rho_f \|A\|_2,\; L_\rvv=\rho_f \|B\|_2$, i.e.
        \begin{align}
            \|\hs_\rvx(t)-\hs_{\rvx'}(t)\|\leq (\rho_f \|B\|_2)\sum_{k=1}^\infty (\rho_f\|A\|_2)^{k-1}  \|\rvx(t-k)-\rvx'(t-k)\|
        \end{align}
        Lastly Lipschitz output property follows by assumption, and similarly we have $L_{g,\rvs}=\rho_g\|C\|_2,$ $L_{g,\rvv}=\rho_g\|D\|_2$
    \end{proof}
\end{proposition}

\subsection{Supporting Lemmas}
\label{sec:lemma:full_gen_sys}
\begin{lemma}[Full generator]\label{lemma:full_gen_sys} Let the generator system $S_{gen}=(\sS_g,\sY\times\sX,\sV,f_g,g_g)$ be of class \class{}, furthermore let the predictor $\hat{S}=(\hat{\sS},\sY,\sX,\hat{f},\hat{g})$ be of class \class{}, then the full generator system $S_{fg}=(\sS_g\times\hat{\sS},\sY\times\sY,\sV,f_{fg},g_{fg})$ is of class \class{}, with associated constants:
    \begin{align}
    \tilde{\tau}&=\max\{\tau_g,\hat{\tau}\},\; G=-2(e\ln\tilde{\tau})^{-1}\\
    C_{fg}&=\sqrt{C_g^2\left (1+ \tilde{\tau}^{-1}\left (G \hat{L}_{\rvv}L_{g_g,\rvs} \right )^2 \right )+\hat{C}^2 }\\ 
    \tau_{fg}&=\sqrt{\tilde{\tau}}\\
    L_{fg,\rvv}&=\sqrt{L_{g,\rvv}^2+ \tilde{\tau}^{-3}\left (\hat{L}_{\rvv}G \max\{L_{g_g,\rvs} L_{g,\rvv} ,L_{g_g,\rvv} \}  \right )^2} \\ 
    L_{fg,g,\rvs}&=L_{\hat{g},\rvs}\\
    L_{fg,g,\rvv}&=\sqrt{L_{\hat{g},\rvv}^2+1}
    \end{align}
\begin{align}
    S_{fg}=\begin{cases}\rvs_{fg}(t+1)=f_{fg}(\rvs_{fg}(t),\rve_g(t))\\ \rvo(t)=\begin{bmatrix} \hat{\rvy}(t) \\ \rvy_{\rve_g}(t)\end{bmatrix}= g_{fg}(\rvs_{fg}(t),\rve_g(t)) \end{cases} =\begin{cases}\rvs_{g,\rve}(t+1)=f_g(\rvs_{g,\rve}(t),\rve_g(t))\\ \begin{bmatrix} \rvy_{\rve_g}(t)\\ \rvx_{\rve_g}(t) \end{bmatrix}=g_g(\rvs_{g,\rve}(t),\rve_g(t))\\
    \hat{\rvs}(t+1)=\hat{f}(\hat{\rvs}(t),\rvx_{\rve_g}(t))\\
    \rvo(t)=\begin{bmatrix} \hat{g}(\hat{\rvs}(t),\rvx_{\rve_g}(t)) \\ \rvy_{\rve_g}(t)\end{bmatrix}
    \end{cases}
\end{align}
\begin{proof}[\textbf{Proof of Lemma \ref{lemma:full_gen_sys}}]
    Consider a predictor $\hat{S}$, and consider a series connection between the generator $S_{gen}$, and modified predictor $\hat{S}'(\hat{\sS},\sY\times \sY,\sY\times\sX,f',g')$, where
    \begin{align}
        \hat{S}'=\begin{cases} \hat{\rvs}(t+1)=f' \left (\hat{\rvs}(t),\begin{bmatrix} \rvy_{\rve_g}(t)\\\rvx_{\rve_g}(t) \end{bmatrix} \right ) = \hat{f} (\hat{\rvs}(t),\rvx_{\rve_g}(t))\\
        \rvq(t)=\begin{bmatrix} \hat{\rvy}(t)\\ \rvy_{\rve_g}(t)\end{bmatrix}=g'\left (\hat{\rvs}(t),\begin{bmatrix} \rvy_{\rve_g}(t)\\\rvx_{\rve_g}(t) \end{bmatrix}  \right )=\begin{bmatrix}\hat{g}(\rvs(t),\rvx_{\rve_g}(t)) \\  \rvy_{\rve_g}(t)\end{bmatrix} \end{cases}
    \end{align}
    Now since, $f'=\hat{f}$, then $C_{S'}=C_{\hat{S}}, \; \tau_{S'}=\tau_{\hat{S}},\; L_{S',\rvv}=L_{\hat{S},\rvv}$. However since $g'\neq \hat{g}$, we need to show that $g'$ is Lipschitz, i.e. with $\vs=\hat{\vs}$, and $\vv=\begin{bmatrix} \vy\\ \vx \end{bmatrix}$
    \begin{align}
        \|g'(\vs_1,\vv_1)-g'(\vs_2,\vv_2)\|=\left \|\begin{bmatrix}\hat{g}(\vs_1,\vx_1) \\  \vy_1\end{bmatrix} -\begin{bmatrix}\hat{g}(\vs_2,\vx_2) \\  \vy_2\end{bmatrix} \right \|\\
        \leq \|\hat{g}(\vs_2,\vx_2)-\hat{g}(\vs_1,\vx_1)\|+ \|\vy_1-\vy_2\|
    \end{align}
    since the predictor is of class \class{},
    \begin{align}
        \|g'(\vs_1,\vv_1)-g'(\vs_2,\vv_2)\|\leq L_{\hat{g},\rvs}\|\vs_2-\vs_1\|+L_{\hat{g},\rvv}\|\vx_2-\vx_1\|+ \|\vy_1-\vy_2\|
    \end{align}
    By Proposition \ref{prop:normDecomp}, we get
    \begin{align}
        \|g'(\vs_1,\vv_1)-g'(\vs_2,\vv_2)\|\leq L_{\hat{g},\rvs}\|\vs_2-\vs_1\|+\sqrt{L_{\hat{g},\rvv}^2+1}\|\vv_2-\vv_1\|
    \end{align}

    Thus $S'$ is of class \class{}, with $L_{g',\rvs}=L_{\hat{g},\rvs}$, and $L_{g',\rvv}=\sqrt{L_{\hat{g},\rvv}^2+1}$

    Then we can apply Lemma \ref{lemma:series}, and obtain the statement of the lemma.
\end{proof}
\end{lemma}

\subsection{Proofs relating to class  \class{}}\label{sec:proofsClass}

\begin{lemma}\label{lemma:series} 
If
$S_1=(\sS_1,\sY_1,\sV,f_1,g_1)$ and $S_2=(\sS_2,\sY,\sY_1,f_2,g_2)$ are \class{}, then
the system $S=(\sS_1 \times \sS_2, \sY,\sV, f,g)$ defined by
\begin{align*}
    S\triangleq \begin{cases} \underbrace{\begin{bmatrix}\vs_1(t+1) \\
    \vs_2(t+1) \end{bmatrix}}_{\vs(t)}
    =\underbrace{\begin{bmatrix} 
     f_1(\vs_1(t),\vv(t)) \\
     f_2(\vs_2(t),g_1(\vy_1(t),\vv(t))
     \end{bmatrix}}_{f(\vs(t),\vv(t))}, \quad 
    \vy(t)=\underbrace{g_2(\vs_2(t),\vv(t))}_{g(\vs(t),v(t))} \end{cases}
\end{align*}
is class \class{}. In addition
    if the associated constants of $S_i$  are $(C_i,\tau_i,L_{i,\vv},L_{g_i,\vs},L_{g_i,\vv})$, $i=1,2$,
    then 
    \newcommand{\St}{S} 
    \begin{align}
    \begin{split}
    & C_{\St}=\sqrt{C_1^2\left (1+ \tilde{\tau}^{-1}\left (G L_{2,\rvv}L_{g_1,\rvs} \right )^2 \right )+C_2^2 }, ~   L_{\St,g,\rvs}=L_{g_2,\rvs},\quad L_{\St,g,\rvv}=L_{g_2,\rvv},  \\
     & L_{\St,\rvv}=\sqrt{L_{1,\rvv}^2+ \tilde{\tau}^{-3}\left (L_{2,\rvv}G \max\{L_{g1,\rvs} L_{1,\rvv} ,L_{g_1,\rvv} \}  \right )^2}, \quad \tau_{\St}=\sqrt{\tilde{\tau}}\\
    \end{split}
    \end{align}
    are the the associated constants of 
    $\mathcal{S}$, 
    with $\tilde{\tau}\triangleq\max\{\tau_1,\tau_2\}$, and $G\triangleq-2(e\ln\tilde{\tau})^{-1}$.
\end{lemma}
\begin{proof}[\textbf{Proof of Lemma \ref{lemma:series}}]
\newcommand{\St}{S_1\rightarrow S_2}
    Let use denote by  $S_1\rightarrow S_2$
    the interconnected system $S$. The latter can be written in the form as \eqref{eq:generalNLsys}, i.e.,
    \begin{align}
        S_1\rightarrow S_2\triangleq &\begin{cases} \begin{bmatrix} \vs_1(t+1)\\\vs_2(t+1) \end{bmatrix}=\begin{bmatrix} f_1(\vs_1(t),\vv(t))\\
        f_2(\vs_2(t),g_1(\vs_1(t),\vv(t)))\end{bmatrix} \\ \rvy_2(t)=g_2(\vs_2(t),g_1(\vs_1(t),\vv(t)))\end{cases}
    \end{align}


    
    

    Next, we will show the existence of a bounded steady-state state trajectory  $\vs_\vv:\sZ\to\sS$.
    
    Since $S_1$ is of class \class{}, thus there exists a unique bounded steady-state state trajectory $\vs_{1,\vv}:\sZ \to \sS_1$, 
    of $S_1$ corresponding to $\vv(t)$. Now let $\vw_\vv(t)=g_1(\vs_{1,\vv}(t),\vv(t))$, and since $S_2$ is of class \class{}, there exists a unique  bounded steady-state state trajectory $\vs_{2,\vw_\vv}$ corresponding to $\vw_\vv$.
    Thus we have the  unique bounded steady-state state trajectory $\vs_\vv(t)$ corresponding to $\vv$, with:
    \begin{align}
        \vs_\vv(t+1)=\begin{bmatrix} \vs_{1,\vv}(t+1) \\ \vs_{2,\vw_\vv}(t+1) \end{bmatrix} =\begin{bmatrix}f_1(\vs_{1,\vv}(t),\vv(t))\\ f_2(\vs_{2,\vv}(t),\vw_\vv(t)) \end{bmatrix}, \quad \text{and } \vw_\vv(t)=g_1(\vs_{1,\vv}(t),\vv(t))
    \end{align}
    
    Next, we will focus on the uniform exponential stability property. Consider
    \begin{equation}
        \vs(\vs_0,t_0,\vv,t)=\begin{bmatrix}
            \vs_1(\vs_{1,0},t_0,\vv,t)\\ 
            \vs_2(\vs_{2,0},t_0,\vw,t)
        \end{bmatrix},\quad \text{with }\vw(t)=g_1(\vs_1(\vs_{1,0},t_0,\vv,t),\vv(t)),\quad \forall t\geq t_0
    \end{equation}
    We need to show that $\exists C_{\St}>0,\; \tau_{\St}\in[0,1)$, s.t.
    \begin{equation}
        \|\vs(\vs_0,t_0,\vv,t)-\vs_\vv(t)\|\leq C_{\St}\tau_{\St}^{t-t_0}\|\vs_0-\vs_\vv(t_0)\|
    \end{equation}

    Since $S_1,S_2$ are  of class \class{}, thus 
    \begin{align}
        \|\vs_1(\vs_{1,0},t_0,\vv,t)-\vs_{1,\vv}(t)\|\leq C_1\tau_1^{t-t_0}\|\vs_{1,0}-\vs_{1,\vv}(t_0)\|\\
        \|\vs_2(\vs_{2,0},t_0,\vw_\vv,t)-\vs_{2,\vw_\vv}(t)\|\leq C_2\tau_2^{t-t_0}\|\vs_{2,0}-\vs_{2,\vw_\vv}(t_0)\|
    \end{align}
    
    Next, we will bound $\|\vs_2(\vs_{2,0},t_0,\vw,t)-\vs_{2,\vw_\vv}(t)\|$. 
    For this 
    let us define\begin{equation}
        \tilde{\vw}(t)\triangleq\begin{cases} \vw(t),& t\geq t_0 \\ \vw_\vv(t),& t<t_0 \end{cases}
    \end{equation}
    Now,
    \begin{align}
        \|\vs_2(\vs_{2,0},t_0,\vw,t)-\vs_{2,\vw_\vv}(t)\|&=\|\vs_2(\vs_{2,0},t_0,\vw,t)-\vs_{2,\tilde{\vw}}(t)+\vs_{2,\tilde{\vw}}(t)-\vs_{2,\vw_\vv}(t)\|\\
        &\leq \|\vs_2(\vs_{2,0},t_0,\vw,t)-\vs_{2,\tilde{\vw}}(t)\|+ \| \vs_{2,\tilde{\vw}}(t)-\vs_{2,\vw_\vv}(t)\| 
    \end{align}
    since for $t\geq t_0:\; \tilde{\vw}(t)=\vw(t)$
    \begin{align}
        \|\vs_2(\vs_{2,0},t_0,\vw,t)-\vs_{2,\vw_\vv}(t)\| \leq \|\vs_2(\vs_{2,0},t_0,\tilde{\vw},t)-\vs_{2,\tilde{\vw}}(t)\|+ \| \vs_{2,\tilde{\vw}}(t)-\vs_{2,\vw_\vv}(t)\| 
    \end{align}
    Now since $S_2$ is exponentially convergent
    $$\|\vs_2(\vs_{2,0},t_0,\tilde{\vw},t)-\vs_{2,\tilde{\vw}}(t)\| \le C_2 \tau_2^{t-t_0} \|\vs_{2,0}-\vs_{2,\tilde{\vw}}(t_0)\|, $$ and 
     by the infinite memory property, 
     $$\| \vs_{2,\tilde{\vw}}(t)-\vs_{2,\vw_\vv}(t)\|  \le L_{2,\rvv}\sum_{k=1}^\infty \tau_2^{k-1}\| \tilde{\vw}(t-k)-\vw_\vv(t-k)\|$$
     Hence, 
    \begin{multline}
        \|\vs_2(\vs_{2,0},t_0,\vw,t)-\vs_{2,\vw_{\vv}}(t)\| \leq C_2\tau_2^{t-t_0} \|\vs_{2,0}-\vs_{2,\tilde{\vw}}(t_0)\|\\
        + L_{2,\rvv}\sum_{k=1}^\infty \tau_2^{k-1}\| \tilde{\vw}(t-k)-\vw_\vv(t-k)\| 
    \end{multline}
    Since $\forall t<t_0:\; \tilde{\vw}(t)=\vw_\vv(t)$, we have $\forall k>t-t_0:\; \tilde{\vw}(t-k)-\vw_\vv(t-k)=0$, thus 
    \begin{multline}
        \|\vs_2(\vs_{2,0},t_0,\vw,t)-\vs_{2,\vw_\vv}(t)\| \leq C_2\tau_2^{t-t_0} \|\vs_{2,0}-\vs_{2,\tilde{\vw}}(t_0)\|\\
        + L_{2,\rvv}\sum_{k=1}^{t-t_0-1} \tau_2^{k-1}\| \vw(t-k)-\vw_\vv(t-k)\| \label{eq:asdlksdf}
    \end{multline}
    Now notice that
    \begin{align}
        \| \vw(t-k)-\vw_\vv(t-k)\| = \| g_1(\vs_1(\vs_{1,0},t_0,\vv,t-k),\vv(t-k)) - g_1(\vs_{1,\vv}(t-k),\vv(t))\|
    \end{align}
    By Lipschitz condition on the output of \class{} class system $S_1$
    \begin{multline}
        \| g_1(\vs_1(\vs_{1,0},t_0,\vv,t-k),\vv(t-k)) - g_1(\vs_{1,\vv}(t-k),\vv(t))\|\\
        \leq L_{g1,\vs}\|\vs_1(\vs_{1,0},t_0,\vv,t-k) -  \vs_{1,\vv}(t-k)\|
    \end{multline}
    By exponential stability of $S_1$
    \begin{align}
        L_{g1,\rvs}\|\vs_1(\vs_{1,0},t_0,\vv,t-k) -  \vs_{1,\vv}(t-k)\| \leq L_{g1,\rvs}C_1\tau_1^{t-k-t_0} \|\vs_{1,0}-\vs_{1,\vv}(t_0)\|
    \end{align}
    Now bringing this back to \eqref{eq:asdlksdf}, we have
    \begin{multline}
        \|\vs_2(\vs_{2,0},t_0,\vw,t)-\vs_{2,\vw_\vv}(t)\| \leq C_2\tau_2^{t-t_0} \|\vs_{2,0}-\vs_{2,\vw}(t_0)\|\\
        + L_{2,\rvv}L_{g1,\rvs}C_1\|\vs_{1,0}-\vs_{1,\vv}(t_0)\|\sum_{k=1}^{t-t_0-1} \tau_2^{k-1}\tau_1^{t-k-t_0} 
    \end{multline}
    now let $\tilde{\tau}=\max(\tau_1,\tau_2)$, then
    \begin{multline}
        \|\vs_2(\vs_{2,0},t_0,\vw,t)-\vs_{2,\vw_\vv}(t)\| 
        \leq C_2\tau_2^{t-t_0} \|\vs_{2,0}-\vs_{2,\vw}(t_0)\|\\
        + L_{2,\rvv}L_{g1,\rvs}C_1(t-t_0-1)\tilde{\tau}^{t-t_0-1} \|\vs_{1,0}-\vs_{1,\vv}(t_0)\|
    \end{multline}
    Note, $(t-t_0-1)\tilde{\tau}^{t-t_0-1}\leq c\sqrt{\tilde{\tau}}^{t-t_0-1}$, for some $c>0$, more precisely, \[c\geq \sup_{t-t_0-1} (t-t_0-1)\sqrt{\tilde{\tau}}^{t-t_0-1}=\frac{-2}{e\ln{b}},\] thus
    \begin{multline}
        \|\vs_2(\vs_{2,0},t_0,\vw,t)-\vs_{2,\vw_\vv}(t)\| \\ 
        \leq C_2\sqrt{\tilde{\tau}}^{t-t_0} \|\vs_{2,0}-\vs_{2,\vw}(t_0)\|+ L_{2,\rvv}L_{g1,\vs}C_1 \frac{-2}{e\ln{\tilde{\tau}}}\sqrt{\tilde{\tau}}^{t-t_0-1} \|\vs_{1,0}-\vs_{1,\vv}(t_0)\|
    \end{multline}
    Now by Proposition \ref{prop:normDecomp}, we have 
    \begin{align}
        \|\vs_2(\vs_{2,0},t_0,\vw,t)-\vs_{2,\vw_\vv}(t)\| 
        &\leq \sqrt{C_2^2+\left (L_{2,\rvv}L_{g1,\rvs}C_1 \frac{-2}{e\sqrt{\tilde{\tau}}\ln{\tilde{\tau}}} \right )^2} \sqrt{\tilde{\tau}}^{t-t_0}\|\vs_0-\vs_\vv(t_0) \|_2\\
        &=\bar{C}\sqrt{\tilde{\tau}}^{t-t_0}\|\vs_0-\vs_\vv(t_0) \|_2,
    \end{align}
    where $\bar{C}=\sqrt{C_2^2+\left (L_{2,\rvv}L_{g1,\rvs}C_1 \frac{-2}{e\sqrt{\tilde{\tau}}\ln{\tilde{\tau}}} \right )^2}$

    
    
    Recall that
    \begin{align}
        \|\vs_1(\vs_{1,0},t_0,\vv,t)-\vs_{1,\vv}(t)\| \leq C_1\tau_1^{t-t_0} \| \vs_{1,0}-\vs_{1,\vv}(t_0) \|
    \end{align}
    since $\tau_1\leq \sqrt{\tilde{\tau}}$, and $\| \vs_{1,0}-\vs_{1,\vv}(t_0) \|\leq \|\vs_0-\vs_\vv(t_0) \|_2$ we have 
    \begin{align}
        \|\vs_1(\vs_{1,0},t_0,\vv,t)-\vs_{1,\vv}(t)\| \leq C_1\sqrt{\tilde{\tau}}^{t-t_0} \|\vs_0-\vs_\vv(t_0) \|_2
    \end{align}
    and since $\|[a,b]\|_2=\sqrt{\|a\|^2+\|b\|^2}$, we have
    \begin{align}
        \left \| \begin{bmatrix} \vs_1(\vs_{1,0},t_0,\vv,t)-\vs_{1,\vv}(t) \\ \vs_2(\vs_{2,0},t_0,\vw,t)-\vs_{2,\vw_\vv}(t) \end{bmatrix} \right \|_2 
        &\leq \sqrt{C_1^2 + \bar{C}^2}\sqrt{\tilde{\tau}}^{t-t_0}\|\vs_0-\vs_\vv(t_0) \|_2
    \end{align}
    
    Thus series connection is exponentially stable with constants: 
    \begin{align}
        \tilde{\tau}&=\max\{\tau_1,\tau_2 \} \\
        \tau_{\St}&=\sqrt{\tilde{\tau}} \\ 
        C_{\St}&=\sqrt{C_1^2\left (1+ \left (\frac{-2L_{2,\rvv}L_{g1,\rvs}}{e\sqrt{\tilde{\tau}}\ln{\tilde{\tau}}} \right )^2 \right )+C_2^2 }
    \end{align}
    
    Finally we need to show that
    \begin{align}
        \|\vs_\vv(t)-\vs_{\bar{\vv}}(t)\|  \leq
        L_{\St,\vv} \sum_{k=1}^\infty \tau_{\St}^{k-1} \| \vv(t-k)-\bar{\vv}(t-k)\|
    \end{align}
    Firstly
    \begin{align}
        \|\vs_\vv(t)-\vs_{\bar{\vv}}(t)\|_2=\left \|\begin{bmatrix}\vs_{1,\vv}(t)-\vs_{1,\bar{\vv}}(t)\\ \vs_{2,\vw_{\vv}}(t)-\vs_{2,\vw_{\bar{\vv}}}(t) \end{bmatrix} \right \|_2,\\
        =\sqrt{\|\vs_{1,\vv}(t)-\vs_{1,\bar{\vv}}(t)\|_2^2+\|\vs_{2,\vw_{\vv}}(t)-\vs_{2,\vw_{\bar{\vv}}}(t)\|_2^2 },
    \end{align}
    where $\vw_\vv(t)=g_1(\vs_{1,\vv}(t),\vv(t))$, and $\vw_{\bar{\vv}}(t)=g_1(\vs_{1,\bar{\vv}}(t),\bar{\vv}(t))$
    
    Since $S_2$ is of class \class{}, then the following holds
    \begin{align}
        \|\vs_{2,\vw_{\vv}}(t)-\vs_{2,\vw_{\bar{\vv}}}(t)\|_2 \leq L_{2,\rvv}\sum_{k=1}^\infty \tau_2^{k-1} \|\vw_\vv(t-k)-\vw_{\bar{\vv}}(t-k)\|_2
    \end{align}
    and since $S_1$ is of class \class{}, we have
    \begin{align}
        \|\vw_\vv(t-k)-\vw_{\bar{\vv}}(t-k)\|_2 \leq L_{g1,\rvs} \|\vs_\vv(t-k)-\vs_{\bar{\vv}}(t-k)\|_2 + L_{g1,\rvv} \|\vv(t-k)-\bar{\vv}(t-k)\|_2\\
        \|\vs_\vv(t-k)-\vs_{\bar{\vv}}(t-k)\|_2 \leq L_{1,\rvv}\sum_{j=1}^\infty \tau_1^{j-1} \|\vv(t-k-j)-\bar{\vv}(t-k-j)\|_2
    \end{align}
    \newcommand{\dv}[1]{\delta\vv(#1)}
    Bringing this back, and for notation let us define $\dv{t}\triangleq \vv(t)-\bar{\vv}(t)$
    \begin{multline}
        \|\vs_{2,\vw_{\vv}}(t)-\vs_{2,\vw_{\bar{\vv}}}(t)\|_2 \leq L_{2,\rvv}\sum_{k=1}^\infty \sum_{j=1}^\infty \Big ( \tau_1^{j-1}\tau_2^{k-1} L_{g1,\rvs} L_{1,\rvv} \|\dv{t-k-j}\|_2  \\ 
        + \tau_2^{k-1} L_{g1,\rvv} \|\dv{t-k}\|_2 \Big )
    \end{multline}
    Let $\tilde{C}=L_{2,\rvv}\max\{L_{g1,\rvs} L_{1,\rvv} ,L_{g1,\rvv} \}$, and $\tilde{\tau}=\max\{\tau_1,\tau_2\}$ then
    \begin{align}
        \|\vs_{2,\vw_{\vv}}(t)-\vs_{2,\vw_{\bar{\vv}}}(t)\|_2 \leq \tilde{C} \left ( \sum_{k=1}^\infty \sum_{j=1}^\infty \tilde{\tau}^{k+j-2}  \|\dv{t-k-j}\|_2  + \sum_{k=1}^\infty \tilde{\tau}^{k-1} \|\dv{t-k}\|_2 \right )
    \end{align}
    Note that, 
    \begin{align}
        \sum_{k=1}^\infty \sum_{j=1}^\infty \tilde{\tau}^{k+j-2}  \|\dv{t-k-j}\|_2 = \sum_{k=1}^\infty \sum_{l=k+1}^\infty \tilde{\tau}^{l-2}  \|\dv{t-l}\|_2
    \end{align}
    Now note that we can switch the order of the sums by
    \begin{align}
        \begin{cases} 1\leq k\leq \infty \\ 1+k\leq l \leq \infty \end{cases} \Rightarrow \begin{cases} 1\leq k\leq l-1 \\ 2 \leq l \leq \infty \end{cases},
    \end{align}
    which yields,
    \begin{align}
     \sum_{k=1}^\infty \sum_{l=k+1}^\infty \tilde{\tau}^{l-2}  \|\dv{t-l}\|_2 =   \sum_{l=2}^\infty \sum_{k=1}^{l-1} \tilde{\tau}^{l-2}  \|\dv{t-l}\|_2 = \sum_{l=2}^\infty (l-1)\tilde{\tau}^{l-2}  \|\dv{t-l}\|_2
    \end{align}
    note that, if $l=1$, then $(l-1)\tilde{\tau}^{l-2}  \|\dv{t-l}\|_2=0$, so we can change the limits on $l$, and with $\tilde{\tau}^{k-1}< \tilde{\tau}^{k-2}$
    \begin{align}
        \|\vs_{2,\vw_{\vv}}(t)-\vs_{2,\vw_{\bar{\vv}}}(t)\|_2 &\leq \tilde{C} \left (\sum_{l=1}^\infty (l-1)\tilde{\tau}^{l-2}  \|\dv{t-l}\|_2  + \tilde{\tau}^{l-2} \|\dv{t-l}\|_2 \right )\\
        &= \tilde{C} \sum_{l=1}^\infty l \tilde{\tau}^{l-2}  \|\dv{t-l}\|_2
    \end{align}
    As before $l\tilde{\tau}^l\leq \frac{-2}{e\ln(\tilde{\tau})}\sqrt{\tilde{\tau}}^l$, and thus
    \begin{align}
        \|\vs_{2,\vw_{\vv}}(t)-\vs_{2,\vw_{\bar{\vv}}}(t)\|_2 \leq \tilde{C} \frac{-2\tilde{\tau}^{-2}}{e\ln(\tilde{\tau})} \sum_{l=1}^\infty \sqrt{\tilde{\tau}}^l  \|\dv{t-l}\|_2 \\
        =\tilde{C} \frac{-2\tilde{\tau}^{-3/2}}{e\ln(\tilde{\tau})} \sum_{l=1}^\infty \sqrt{\tilde{\tau}}^{l-1}  \|\vv(t-l)-\bar{\vv}(t-l)\|_2,
    \end{align}
    with $\left (\tilde{C} \frac{-2\tilde{\tau}^{-3/2}}{e\ln(\tilde{\tau})}\right )=L_{2,\rvv}\frac{-2\tilde{\tau}^{-3/2}}{e\ln(\tilde{\tau})} \max\{L_{g1,\rvs} L_{1,\rvv} ,L_{g1,\rvv} \} $
    
    Now since, $S_1$ is of class \class{}, then
    \begin{align}
        \|\vs_{1,\vv}(t)-\vs_{1,\bar{\vv}}(t)\| \leq L_{1,\rvv}\sum_{k=1}^\infty \tau_1^{k-1}\|\vv(t-l)-\bar{\vv}(t-l)\|
    \end{align}
    with $\bar{L}= L_{2,\rvv}\frac{-2\tilde{\tau}^{-3/2}}{e\ln(\tilde{\tau})} \max\{L_{g1,\rvs} L_{1,\rvv} ,L_{g1,\rvv} \} $, and $\tau_1\leq \sqrt{\tilde{\tau}}$ we also have
    \begin{align}
        \|\vs_{1,\vv}(t)-\vs_{1,\bar{\vv}}(t)\| \leq L_{1,\rvv}\sum_{k=1}^\infty \sqrt{\tilde{\tau}}^{k-1}\|\vv(t-l)-\bar{\vv}(t-l)\|\\
        \|\vs_{2,\vw_{\vv}}(t)-\vs_{2,\vw_{\bar{\vv}}}(t)\|_2 \leq \bar{L}\sum_{k=1}^\infty \sqrt{\tilde{\tau}}^{k-1}\|\vv(t-l)-\bar{\vv}(t-l)\|
    \end{align}

    Recall,
    \begin{align}
        \|\vs_\vv(t)-\vs_{\bar{\vv}}(t)\|_2= \sqrt{\|\vs_{1,\vv}(t)-\vs_{1,\bar{\vv}}(t)\|_2^2+\|\vs_{2,\vw_{\vv}}(t)-\vs_{2,\vw_{\bar{\vv}}}(t)\|_2^2 },\\
        \leq \sqrt{L_{1,\rvv}^2+\bar{L}^2}\sum_{k=1}^\infty \sqrt{\tilde{\tau}}^{k-1}\|\vv(t-l)-\bar{\vv}(t-l)\|
    \end{align}
    With this we have shown that a series connection of two class \class{} systems, yields a class \class{} system with associated constants $(C,\tau,L_\vv,L_{g,x},L_{g,\rvv})$:
    \begin{align}
        \tilde{\tau}&=\max\{\tau_1,\tau_2\}\\
        C_{\St}&=\sqrt{C_1^2\left (1+ \left (\frac{-2L_{2,\rvv}L_{g1,\rvs}}{e\sqrt{\tilde{\tau}}\ln{\tilde{\tau}}} \right )^2 \right )+C_2^2 } \\
        \tau_{\St}&=\sqrt{\tilde{\tau}}\\
        L_{\St,\rvv}&=\sqrt{L_{1,\rvv}^2+\left (L_{2,\rvv}\frac{-2\tilde{\tau}^{-3/2}}{e\ln(\tilde{\tau})} \max\{L_{g1,\rvs} L_{1,\rvv} ,L_{g1,\rvv} \}  \right )^2}\\
        L_{\St,g,\rvs}&=L_{g2,\rvs}\\
        L_{\St,g,\rvv}&=L_{g2,\rvv}
    \end{align}

    
\end{proof}

\subsection{Proofs of Main Results}
\label{sec:proofsResults}

\begin{lemma}[Prediction Error generator]\label{lemma:err_sys} Let the generator system $S_{gen}=(\sS_g,\sY\times\sX,\sV,f_g,h_g)$ be of class \class{}, and have associated constants $(C_g,\tau_g,L_{g,\rvv},L_{g_g,\rvs},L_{g_g,\rvv})$, furthermore let the predictor $\hat{S}=(\hat{\sS},\sY,\sX,\hat{f},\hat{g})$ be of class \class{}, and have associated constants $(\hat{C},\hat{\tau},\hat{L}_{\rvv},L_{\hat{g},\rvs},L_{\hat{g},\rvv})$, then the error system $S_e=(\sS_g\times\hat{\sS},\sY,\sV,f_e,g_e)$ is of class \class{}, with associated constants:
    \begin{align}
    \tilde{\tau}&=\max\{\tau_g,\hat{\tau}\},\; G=-2(e\ln\tilde{\tau})^{-1}\\
    C_{e}&=\sqrt{C_g^2\left (1+ \tilde{\tau}^{-1}\left (G \hat{L}_{\rvv}L_{g_g,\rvs} \right )^2 \right )+\hat{C}^2 }\\ 
    \tau_{e}&=\sqrt{\tilde{\tau}}\\
    L_{e,\rvv}&=\sqrt{L_{g,\rvv}^2+ \tilde{\tau}^{-3}\left (\hat{L}_{\rvv}G \max\{L_{g_g,\rvs} L_{g,\rvv} ,L_{g_g,\rvv} \}  \right )^2} \\ 
    L_{e,g,\rvs}&=L_{\hat{g},\rvs}\\
    L_{e,g,\rvv}&=\sqrt{L_{\hat{g},\rvv}^2+1}
    \end{align}
where 
\begin{align}
    S_e=\begin{cases}\rvs_{g,\rve_g}(t+1)=f_g(\rvs_{g,\rve_g}(t),\rve_g(t))\\ \begin{bmatrix} \rvy_{\rve_g}(t)\\ \rvx_{\rve_g}(t) \end{bmatrix}=h_g(\rvs_{g,\rve_g}(t),\rve_g(t))\\
    \hat{\rvs}(t+1)=\hat{f}(\hat{\rvs}(t),\rvx_{\rve_g}(t))\\
    \rvepsilon(t)=\rvy_{\rve_g}(t)-\hat{\rvy}(t)=\rvy_{\rve_g}(t)-\hat{g}(\hat{\rvs}(t),\rvx_{\rve_g}(t))
    \end{cases}
\end{align}
\begin{proof} Consider a predictor $\hat{S}$, and consider a series connection between the generator $S_{gen}$, and modified predictor $\hat{S}'\sim(f',h')$, where
    \begin{align}
        \hat{S}'=\begin{cases} \hat{\rvs}(t+1)=f' \left (\hat{\rvs}(t),\begin{bmatrix} \rvy_{\rve_g}(t)\\\rvx_{\rve_g}(t) \end{bmatrix} \right ) = \hat{f} (\hat{\rvs}(t),\rvx_{\rve_g}(t))\\
        \rvepsilon(t)=\rvy_{\rve_g}(t)-\hat{\rvy}(t)=h'\left (\hat{\rvs}(t),\begin{bmatrix} \rvy_{\rve_g}(t)\\\rvx_{\rve_g}(t) \end{bmatrix}  \right )=\rvy_{\rve_g}(t)-\hat{g}(\rvs(t),\rvx_{\rve_g}(t)) \end{cases}
    \end{align}
Now since, $f'=\hat{f}$, then $C_{S'}=C_{\hat{S}}, \; \tau_{S'}=\tau_{\hat{S}},\; L_{S',\rve_g}=L_{\hat{S},\rve_g}$. However since $h'\neq h$, we need to show that $h'$ is Lipschitz, i.e. with $\vs=\hat{\vs}$, and $\vv=\begin{bmatrix} \vy\\ \vx \end{bmatrix}$
    \begin{align}
        \|h'(\vs_1,\vv_1)-h'(\vs_2,\vv_2)\|=\|\vy_1-\hat{g}(\vs_1,\vx_1)-\vy_2+\hat{g}(\vs_2,\vx_2)\|\\
        \leq \|\hat{g}(\vs_2,\vx_2)-\hat{g}(\vs_1,\vx_1)\|+ \|\vy_1-\vy_2\|
    \end{align}
since the predictor is of class \class{},
\begin{align}
    \|h'(\vs_1,\vv_1)-h'(\vs_2,\vv_2)\|\leq L_{\hat{g},\rvs}\|\vs_2-\vs_1\|+L_{\hat{g},\rvv}\|\vx_2-\vx_1\|+ \|\vy_1-\vy_2\|
\end{align}
By Proposition \ref{prop:normDecomp}, we get
\begin{align}
        \|h'(\vs_1,\vv_1)-h'(\vs_2,\vv_2)\|\leq L_{\hat{g},\rvs}\|\vs_2-\vs_1\|+\sqrt{L_{\hat{g},\rvv}^2+1}\|\vv_2-\vv_1\|
\end{align}

Thus $S'$ is of class \class{}, with $L_{g',\rvs}=L_{\hat{g},\rvs}$, and $L_{g',\rvv}=\sqrt{L_{\hat{g},\rvv}^2+1}$

Then we can apply Lemma \ref{lemma:series}, and obtain the statement of the lemma.

\end{proof}
    
\end{lemma}

\begin{proof}[\textbf{Proof of Theorem \ref{thm:mainThm}}]
    We apply Lemma \ref{lemma:PAC-BayesianKL_General} on $X(h)=\mathcal{L}(\theta)$, and $Y(h)=V_N(\theta)$, and with probability at least $1-\delta$ we have
    \begin{align}
         \forall\rho\in\mathcal{M}_\pi:\quad &E_{\theta\sim \hat{\rho}} \mathcal{L} (\theta) \le \  E_{\theta\sim \hat{\rho}} V_N(\theta) +\dfrac{1}{\lambda}\!\left[KL(\hat{\rho} \|\pi) + \ln\dfrac{1}{\delta}	+ \hat{\Psi}_{1,\pi}(\lambda,N) \right ],
    \end{align}
    with $\hat{\Psi}_{1,\pi}(\lambda,N)\geq\ln E_{\theta\sim\pi} \E[e^{\lambda(\mathcal{L}(\theta)-V_N(\theta))}]$, by Lemma \ref{lemma:mgf(L-V)}, we can take
    \begin{align}
        \hat{\Psi}_{1,\pi}(\lambda,N)\geq \ln E_{\theta\sim\pi} \exp\left \{\frac{\lambda^2 L_\ell^2}{2N}\left ((B^\rvq+\bar{\theta}_{\infty}^{\rvq}) G_\theta+ B^q H_\theta \right )^2 \right \}
    \end{align}
    where $G_\theta=\frac{\hat{L}_{g,\rvs}\hat{L}_\rvv}{1-\hat{\tau}}+\hat{L}_{g,\rvv} $, and $H_\theta=\left ( \frac{\hat{L}_{g,\rvs}\hat{L}_\rvv}{(1-\hat{\tau})^2} \right )$
    Now we again apply Lemma \ref{lemma:PAC-BayesianKL_General}, with $X(h)=V_N(\theta)$, and $Y(h)=\hat{\mathcal{L}}_N(\theta)$ and with probability at least $1-\delta$ we have
    \begin{align}
         \forall\rho:\quad &E_{\theta\sim \hat{\rho}} V_N (\theta) \le \  E_{\theta\sim \hat{\rho}} \hat{\mathcal{L}}_N(\theta) +\dfrac{1}{\lambda}\!\left[KL(\hat{\rho} \|\pi) + \ln\dfrac{1}{\delta}	+ \hat{\Psi}_{2,\pi}(\lambda,N) \right ],
    \end{align}
    with $\hat{\Psi}_{2,\pi}(\lambda,N)\geq \ln E_{\theta\sim\pi} \E[e^{\lambda|V_N(\theta)-\hat{\mathcal{L}}_N(\theta)|}] \geq\ln E_{\theta\sim\pi} \E[e^{\lambda(V_N(\theta)-\hat{\mathcal{L}}_N(\theta))}]$, by Lemma \ref{lemma:mgf(V-L_hat)}, we have
    \begin{align}
        \hat{\Psi}_{2,\pi}(\lambda,N)\triangleq \ln E_{\theta\sim\pi} \exp \left \{ \frac{\lambda L_\ell C(\theta)}{N} \left ( 2 B^\rvq H_\theta+\|\hat{\vs}_{s,\theta}\|\frac{L_{g,\rvs}(\theta) }{1-\tau(\theta)} \right ) \right \}
    \end{align}

Now by union bound, we have with probability $1-2\delta$
\begin{align}
    \forall\rho:\quad &E_{\theta\sim \hat{\rho}} \mathcal{L} (\theta) \le \  E_{\theta\sim \hat{\rho}} \hat{\mathcal{L}}_N(\theta) +\dfrac{2}{\lambda}\!\left[KL(\hat{\rho} \|\pi) + \ln\dfrac{1}{\delta}	+ \frac{1}{2}\left ( \hat{\Psi}_{1,\pi}(\lambda,N)+\hat{\Psi}_{2,\pi}(\lambda,N)  \right )\right ] 
\end{align}
Let us define $\tilde{\lambda}\triangleq 0.5\lambda$, then we have with probability $1-2\delta$
\begin{align}
    \forall\rho:\quad &E_{\theta\sim \hat{\rho}} \mathcal{L} (\theta) \le \  E_{\theta\sim \hat{\rho}} \hat{\mathcal{L}}_N(\theta) +\dfrac{1}{\tilde{\lambda}}\!\left[KL(\hat{\rho} \|\pi) + \ln\dfrac{1}{\delta}	+ \frac{1}{2}\left ( \hat{\Psi}_{1,\pi}(2\tilde{\lambda},N)+\hat{\Psi}_{2,\pi}(2\tilde{\lambda},N)  \right )\right ],
\end{align}
with $\hat{\Psi}_{\pi}(\tilde{\lambda},N)\triangleq \frac{1}{2}\left ( \hat{\Psi}_{1,\pi}(2\tilde{\lambda},N)+\hat{\Psi}_{2,\pi}(2\tilde{\lambda},N)  \right )$, now let us redefine some quantities for the sake of notation. 
\begin{equation}
    \hat{\Psi}_{\pi}(\tilde{\lambda},N)=\frac{1}{2} \left ( \ln E_{\theta\sim\pi}\widehat{\Psi}_{1,\theta} +\ln E_{\theta\sim\pi}\widehat{\Psi}_{2,\theta} \right )
\end{equation}
with
\begin{align}
    \widehat{\Psi}_{1,\theta}\triangleq \exp\left \{\frac{2\tilde{\lambda}^2 L_\ell^2}{N}\left ((B^\rvq+\bar{\theta}_{\infty}^{\rvq}) \left ( \frac{\hat{L}_{g,\rvs}\hat{L}_\rvv}{1-\hat{\tau}}+\hat{L}_{g,\rvv} \right )+ B^q \left ( \frac{\hat{L}_{g,\rvs}\hat{L}_\rvv}{(1-\hat{\tau})^2} \right ) \right )^2 \right \}\\
    \widehat{\Psi}_{2,\theta}\triangleq \exp \left \{\frac{2\tilde{\lambda} L_\ell}{N} \left ( G_{1,\rvx}(\theta) \|\rvx_\rvv(0)\|_\infty + G_{1,\vs_0}(\theta)\|\hat{\vs}_0\| \right )\right \}
\end{align}
note that $\|\rvx_\rvv(0)\|_\infty \leq B^q$, with which we obtain the statement of the theorem.
\end{proof}

\subsection{Supporting relations}
\begin{proposition}\label{prop:normDecomp} Let $K_1,K_2\geq 0$, then
\begin{align}
    \sqrt{K_1^2+K_2^2} \left \|\begin{bmatrix} a\\b \end{bmatrix} \right \| \geq K_1\|a\|+K_2\|b\|
\end{align}

\begin{proof}

    Consider $C>0$, s.t. for some $K>0$
    \begin{align}
        C \left \|\begin{bmatrix} a\\b \end{bmatrix} \right \|= C \sqrt{\|a\|^2+\|b\|^2} \geq K\|a\|+\|b\|\\
        C\geq \frac{ K\|a\|+\|b\|}{\sqrt{\|a\|^2+\|b\|^2}}
    \end{align}
    we can perform a change in variables, $\|a\|=r\sin{\theta},\;\|b\|=r\cos{\theta},\; \forall r>0,\; 0<=\theta<=\frac{\pi}{2}$, now without loss of generality we have
        \begin{align}
        C\geq K\sin{\theta}+\cos{\theta}
    \end{align}
    since $\sin{\theta}\leq 1$, and $\cos{\theta}\leq 1$, we could conclude that $C\geq K+1$, but taking $\frac{\delta}{\delta \theta } (K\sin{\theta}+\cos{\theta})=0$, we can find a tighter bound
    \begin{align}
        \frac{\delta}{\delta \theta } (K\sin{\theta}+\cos{\theta}) = K\cos{\theta}-\sin(\theta)=0 \Rightarrow \frac{\sin{\theta}}{\cos{\theta}}=K\\
        \Rightarrow \theta=\arctan{K}\\
        C\geq K\sin{\arctan{K}}+\cos{\arctan{K}}\\
        C\geq \frac{K^2}{\sqrt{K^2+1}}+\frac{1}{\sqrt{K^2+1}} = \sqrt{K^2+1}
    \end{align}
    now let $K=\frac{K_1}{K_2}$, and we have for $K_2\neq 0$
    \begin{align}
        \sqrt{\frac{K_1^2}{K_2^2}+1}\left \|\begin{bmatrix} a\\b \end{bmatrix} \right \| \geq \frac{K_1}{K_2}\|a\|+\|b\|\\
        \Rightarrow \sqrt{K_1^2+K_2^2}\left \|\begin{bmatrix} a\\b \end{bmatrix} \right \| \geq K_1 \|a\|+K_2\|b\|
    \end{align}
    and when $K_2=0$, we have 
    \begin{align}
        C\sqrt{\|a\|^2+\|b\|^2}\geq K_1\|a\|
    \end{align}
    thus with $C=K_1=\sqrt{K_1^2+0^2}$, the above holds. 
\end{proof}
\end{proposition}

\begin{lemma}\label{lemma:Cross-EntropyLipschitz} Cross-Entropy Loss with soft-max is Lipschitz, for $\rvy\in[0,1]^K$, and bounded $\hat{\rvy}$ (e.g. outputs of class \class{} system) , i.e.
\begin{align}
    \ell\left (\begin{bmatrix} \rvy(t)\\\hat{\rvy}(t) \end{bmatrix} \right )=-\sum_{i=1}^K \rvy_i(t)\ln \left ( \frac{e^{\hat{\rvy}_i(t)}}{\sum_{j=1}^K e^{\hat{\rvy}_j(t)} }\right )
\end{align}
with $\rvq(t)=\begin{bmatrix} \rvy(t)\\\hat{\rvy}(t) \end{bmatrix}$
    \begin{align}
        \|\ell(\rvq)-\ell(\rvq')\|\leq K(2\|\hat{\rvy}\|_\infty + 2 + \ln{K})\|\rvq-\rvq'\|
    \end{align}
\begin{proof}
We will use the fact that
\begin{align}
    Lip(\ell)=\sup_{\rvy,\hat{\rvy}}\sup_{\|v\|_1=1}|D\ell(\rvy,\hat{\rvy})\cdot v|
\end{align}
with $D\ell(\rvy,\hat{\rvy})=\left [\frac{\partial \ell(\rvy,\hat{\rvy})}{\partial \rvy_1},\dots,\frac{\partial \ell(\rvy,\hat{\rvy})}{\partial \rvy_K},\frac{\partial \ell(\rvy,\hat{\rvy})}{\partial \hat{\rvy}_1},\dots, \frac{\partial \ell(\rvy,\hat{\rvy})}{\partial \hat{\rvy}_K} \right ]$
Note that
\begin{align}
    \frac{\partial \ell(\rvy,\hat{\rvy})}{\partial \rvy_i}=-\hat{\rvy}_i+\ln\left(\sum_{j=1}^K e^{\hat{\rvy}_j} \right ),\\
    \frac{\partial \ell(\rvy,\hat{\rvy})}{\partial \hat{\rvy}_i}=-\rvy_i+\frac{e^{\hat{\rvy}_i}}{\sum_{j=1}^K e^{\hat{\rvy}_j}} \sum_{p=1}^K \rvy_p
\end{align}
with this we get
\begin{align}
    D\ell(\rvy,\hat{\rvy})=-[\hat{\rvy}^T,\rvy^T]+\left [\ln\left(\sum_{j=1}^K e^{\hat{\rvy}_j} \right )\mathbf{1}_k, \frac{\sum_{p=1}^K \rvy_p}{\sum_{j=1}^K e^{\hat{\rvy}_j}} \exp(\hat{\rvy}) \right ]
\end{align}
where $\mathbf{1}_k\in \reals^k$ is a vector of ones, and $\exp(\hat{\rvy})=[e^{\hat{\rvy}_1},\dots, e^{\hat{\rvy}_K}]$.
Now
\begin{align}
    \sup_{\rvy,\hat{\rvy}}\sup_{\|v\|_1=1}|D\ell(\rvy,\hat{\rvy})\cdot v|\leq  \sup_{\rvy,\hat{\rvy}}\sup_{\|v\|_1=1}|[\hat{\rvy}^T,\rvy^T]v|+\left |\left [\ln\left(\sum_{j=1}^K e^{\hat{\rvy}_j} \right )\mathbf{1}_k, \frac{\sum_{p=1}^K \rvy_p}{\sum_{j=1}^K e^{\hat{\rvy}_j}} \exp(\hat{\rvy}) \right ]v \right |
\end{align}
Now we will consider some assumptions, firstly $\hat{\rvy}_i\leq \|\hat{\rvy}\|_\infty$, also the class labels $\rvy_i\in [0,1]$, i.e. $\rvy_i\leq 1$, now we could consider that class labels are normalised s.t. $\sum_{p=1}^K \rvy_p=1$, but we will not assume that for the sake of generalisation. Furthermore we upper bound $v_i\leq 1$, since all quantities are positive
with these assumptions we get 
\begin{align}
    |[\hat{\rvy}^T,\rvy^T]v| \leq \sum_{i=1}^K \hat{\rvy}_i + \sum_{i=1}^K \rvy_i \leq K\|\hat{\rvy}\|_\infty + K\\
    |\ln\left(\sum_{j=1}^K e^{\hat{\rvy}_j} \right )\mathbf{1}_k v_{1:k}| \leq  K(\ln{K}+\|\hat{\rvy}\|_\infty)  \\
    |\frac{\sum_{p=1}^K \rvy_p}{\sum_{j=1}^K e^{\hat{\rvy}_j}} \exp(\hat{\rvy}) v_{k+1:2k} | \leq |\sum_{p=1}^K \rvy_p| \leq K
\end{align}
where $v_{1:k}$ denotes the first $k$ components of $v$, and $v_{k+1:2k}$ denotes the last $k$ components of $v$, with the above we obtain 
\begin{align}
    \sup_{\rvy,\hat{\rvy}}\sup_{\|v\|_1=1}|D\ell(\rvy,\hat{\rvy})\cdot v|\leq K(2\|\hat{\rvy}\|_\infty + 2 + \ln{K})
\end{align}

\end{proof}
    
\end{lemma}

\section{Numerical Example:} \label{sec:exampleAppendix} 
The system used in the numerical example is:
\renewcommand{\Ag}{\begin{bmatrix} 0.52&0.23\\0.23&-0.52\end{bmatrix}}
\renewcommand{\Bg}{\begin{bmatrix} -0.82&-0.45\\0.36&-0.96\end{bmatrix}}
\renewcommand{\bxg}{\begin{bmatrix} 0.38\\-0.06\end{bmatrix}}
\renewcommand{\Cg}{\begin{bmatrix} 0.05&-0.10\\-0.11&0.01\end{bmatrix}}
\renewcommand{\Dg}{\begin{bmatrix} 0.09&-0.11\\0.05&-0.16\end{bmatrix}}
\renewcommand{\byg}{\begin{bmatrix} -0.53\\-0.79\end{bmatrix}}
\renewcommand{\ce}{1.27}
In this section we shall explore a synthetic example to illustrate the proposed bound. The code for this example is provided in \cite{NLPACBayesCode}.
We shall generate synthetic data for regression, in order to do so we need a data generator that satisfies Assumption \ref{as:generator}. We randomly chose a generator with $B^\rvq=\sqrt{2},\;\bar{\theta}_{\infty,N}(1)=2$, and $\|\rve_g(t)\|_\infty\leq \ce$:
\begin{subequations}\label{eq:numexGen}
\begin{align}
\rvs_g(t+1)&=\text{ReLu}\left (\Ag\rvs_g(t)+\Bg\rve_g(t)+\bxg \right)\\
\begin{bmatrix} \rvy(t)\\ \rvx(t)\end{bmatrix}&=\text{tanh}\left (\Cg\rvs_g(t)+\Dg\rve_g(t)+\byg \right )
\end{align}
\end{subequations}

The figure and results in Section \ref{sec:example}, have been computed using Markov Chain-Monte Carlo method. That is, we need to define a function that is proportional to the prior, e.g. $\pi(\theta)\propto \hat{\pi}(\theta)= \exp\{ -\frac{1}{2}\theta^T\Sigma^{-1}\theta \}$, furthermore for numerical stability standard MCMC algorithms accept the distributions in log domain, that is in code we define
\begin{equation}
    \textbf{Step 0. Define:}\quad \ln\hat{\pi}(\theta)=-\frac{1}{2}\theta^T\Sigma^{-1}\theta \label{ap:eq:example:prior}
\end{equation}
Now using standard algorithms (E.g. Metropolis Hastings algorithm) and \eqref{ap:eq:example:prior}, we can easily (and relatively quickly) obtain $N_f$ samples from the prior.
\begin{equation}
    \textbf{Step 1. Sample (using Metropolis Hastings):}\quad \sQ=\{\theta_i\}_{i=1}^{N_f},\; \theta_i\sim\pi(\theta),\; \forall i
\end{equation}
We will be using the Gibbs posterior $\rho_N(\theta)\propto \hat{\rho}_N(\theta)=\pi(\theta)\exp\{-\lambda_N\hat{\mathcal{L}}_N(\theta)\}$, however the following is easily extended to any prior and posterior. It turns out, that we don't have to evaluate the posterior distribution, it is enough to be able to evaluate the prior and $\beta(\theta)=\frac{\hat{\rho}_N(\theta)}{\hat{\pi}(\theta)}\propto \frac{\rho(\theta)}{\pi(\theta)}$
\begin{equation}
    \textbf{Step 2. Evaluate } \beta(\theta)\textbf{:},\quad \beta_i=\frac{\hat{\rho}_N(\theta_i)}{\hat{\pi}(\theta_i)}=\exp\{-\lambda_N\hat{\mathcal{L}}_N(\theta_i)\},\quad \forall \theta_i\in \sQ
\end{equation}
Note that in this step, we need to evaluate empirical loss, which is a relatively expensive operation. However, all that is left is simple calculations, firstly compute the following means (in the code provided in \cite{NLPACBayesCode} they are computed recursively):\\
\textbf{Step 3. Compute means:}
\begin{align}
    \hat{Z}^{-1}&\approx \frac{1}{N_f}\sum_{i=1}^{N_f}\beta_i\\
    K &\approx \frac{1}{N_f}\sum_{i=1}^{N_f}\beta_i\ln(\beta_i)\\
    V&\approx \frac{1}{N_f}\sum_{i=1}^{N_f} \beta_i\hat{\mathcal{L}}_N(\theta_i)\\
    \phi_1 &\approx \ln \left (\frac{1}{N_f}\sum_{i=1}^{N_f} \exp \left \{ \frac{2\lambda^2}{N}L_\ell^2(\theta_i)\left [(B^\rvq+\bar{\theta}_\infty(1))G(\theta_i)+B^\rvq H(\theta_i)\right ]^2 \right \} \right )\\
    \phi_2 &\approx \ln \left (\frac{1}{N_f}\sum_{i=1}^{N_f} \exp \left \{ \frac{2\lambda}{N}L_\ell(\theta_i)C(\theta_i) \left (B^\rvq H(\theta_i)+\|\hat{\vs}_{0,\theta}\|\frac{L_{g,\rvs}(\theta_i) }{1-\tau(\theta_i)} \right ) \right \} \right )
\end{align}
Note that the two terms $\phi_1$ and $\phi_2$, can be better computed using log-sum-exp trick, i.e. $\ln(\sum_i e^{x_i})=\bar{x}+\ln(\sum_i e^{x_i-\bar{x}})$, where $\bar{x}=\max_i x_i$
Finally the last step,\\
\textbf{Step 3. Compute bound:}
\begin{align}
    KL(\rho|\pi)&\approx \ln(\hat{Z})+\hat{Z}K\\
    \widehat{\Psi}&\approx\frac{1}{2}(\phi_1+\phi_2)\\
    r_n&\approx\frac{1}{\lambda}\left (KL(\rho|\pi)+\ln\frac{1}{\delta} + \widehat{\Psi}\right )\\
    \E_{\theta\sim\rho}\hat{\mathcal{L}}_N(\theta)&\approx \hat{Z}V
\end{align}
This algorithm may have higher monte carlo error since expectations over posterior is approximated by samples from the prior, however it is much faster since one does not have to sample from the posterior, where evaluation of the empirical loss is necessary. We can get away with this 'trick' since the requirement of absolutely continuous posterior is already there from the Theorem \ref{thm:mainThm}.  One still needs to compute the empirical loss for the obtained samples from MCMC algorithm, but we avoid evaluating empirical loss for the rejected samples, or samples we skip in the process of 'thinning' in the MCMC algorithm. 

\end{document}